\documentclass[twoside,11pt]{article}

\usepackage{jmlr2e}
\usepackage{bbm}
\usepackage{bm}
\usepackage{amsmath}
\usepackage{amssymb}
\usepackage{mathrsfs}
\usepackage{graphicx}
\usepackage{algorithm}
\usepackage{algorithmic}
\usepackage{subfigure}
\usepackage{epsfig}
\usepackage{epsf}
\usepackage{graphics}
\usepackage{subfigure}
\usepackage{url}

\DeclareMathOperator{\argmax}{argmax}
\newtheorem{them}{Theorem}
\newtheorem{prop}{Proposition}

\newtheorem{assm}{Assumption}

\def\o{{\bf o}}

\def\R {\mathbb{R}}

\def\l{\lambda}
\def\D{\Delta}
\def\XX{{\cal X}}

\def\M{{\cal M}}
\def\bM{{\bf M}}

\def\bH {{\bf H}}
\def\X{{\bf X}}

\def\YY {{\cal Y}}

\def\g{\gamma}

\def\A{{\cal A}}
\def\I{{\bf I}}

\def\eps{\epsilon}

\def\psd{\succeq 0}

\def\K{{\bf K}}
\def\a{\alpha}

\def\b{\beta}
\def\ba{\boldsymbol{\alpha}}
\def\bmu{\boldsymbol{\mu}}
\def\bxi{\boldsymbol{\xi}}

\def\l{\lambda}
\def\th{\theta}
\def\md{{\mathcal{D}}}

\def\I{{\bf I}}
\def\v{{\bf v}}

\def\w{{\bf w}}

\def\x{{\bf x}}
\def\y{{\bf y}}

\def\0{{\bf 0}}
\def\1{{\bf 1}}

\def\pd{\succ 0}
\def\psd{\succeq 0}
\def\A{{\cal A}}
\def\I{{\bf I}}

\def\x{{\bf x}}
\def\0{{\bf 0}}
\def\1{{\bf 1}}

\def\x{{\bf x}}

\def\B{{\cal B}}
\def\C{{\cal C}}

\def\bB{{\bf B}}
\def\dd{{\bf d}}
\def\ss{{\bf s}}
\def\mL{{\mathcal L}}
\def\mU{{\mathcal U}}
\def\bpsi{\boldsymbol{\psi}}

\def\br{{\bf r}}
\def\bt{\boldsymbol{\tau}}
\def\bW{{\bf W}}

\newcommand{\lgsvm}[0]{\textsc{WellSVM}}

\jmlrheading{14}{2013}{2151-2188}{6/12; Revised 4/13}{7/13}{Yu-Feng Li, Ivor Tsang, James Kwok and Zhi-Hua Zhou}

\ShortHeadings{Convex and Scalable Weakly Labeled SVMs}{Li, Tsang, Kwok and Zhou} \firstpageno{2151}

\begin{document}

\title{Convex and Scalable Weakly Labeled SVMs}

\author{\name Yu-Feng Li \email liyf@lamda.nju.edu.cn \\
       \addr National Key Laboratory for Novel Software Technology\\
       Nanjing University\\
       Nanjing 210023, China
       \AND
       \name Ivor W. Tsang \email IvorTsang@ntu.edu.sg \\
       \addr School of Computer Engineering\\
       Nanyang Technological University\\
       Singapore 639798
       \AND
       \name James T. Kwok \email jamesk@cse.ust.hk \\
       \addr Department of Computer Science and Engineering\\
       Hong Kong University of Science \& Technology\\
       Hong Kong
       \AND
       \name Zhi-Hua Zhou \email zhouzh@lamda.nju.edu.cn \\
       \addr National Key Laboratory for Novel Software Technology\\
       Nanjing University\\
       Nanjing 210023, China}

\editor{Sathiya Keerthi}

\maketitle

\begin{abstract}
In this paper, we study the problem of
learning from \emph{weakly labeled data}, where
labels of the training examples are incomplete.
This includes, for example, (i) semi-supervised learning where labels are
partially known; (ii) multi-instance learning where labels are implicitly known;
and (iii) clustering where labels are completely unknown. Unlike
supervised learning, learning with weak labels involves a difficult
Mixed-Integer Programming (MIP) problem. Therefore, it can suffer from poor
scalability and may also get stuck in local minimum. In this paper, we focus on SVMs
and propose the \lgsvm\ via a novel \emph{label generation} strategy.
This leads to a convex relaxation of the original MIP, which
is at least as tight as existing convex Semi-Definite Programming (SDP)  relaxations.
Moreover, the
\lgsvm\ can be solved via a sequence of SVM subproblems that are much more scalable than previous convex SDP relaxations.
Experiments on three weakly labeled learning tasks, namely, (i) semi-supervised
learning; (ii) multi-instance learning for locating regions of interest in
content-based information retrieval; and (iii) clustering, clearly demonstrate
improved performance, and \lgsvm\ is also readily applicable on large data sets.
\end{abstract}


\section{Introduction}

Obtaining labeled data is expensive and difficult.
For example, in scientific applications, obtaining the labels
involves repeated experiments that may be hazardous; in drug prediction,
deriving active molecules of a new drug involves expensive   expertise that
may not even be available.
On the other hand, \emph{weakly labeled data}, where the labels are
incomplete, are often ubiquitous in many applications.
Therefore,
exploiting weakly labeled training data may help improve performance and discover the
underlying structure of the data. Indeed,
this has been regarded as one of the most challenging tasks in machine learning research \citep{mitchell2006discipline}.

Many weak-label learning problems have been proposed. In the following, we summarize several major learning paradigms with weakly labeled data:
\begin{itemize}

\item \emph{Labels are partially known.} \hspace{-.1in} A representative example is
semi-supervised learning (SSL) \citep{chapelle2006semi,zhu2006semi,Zhou:Li2010}, where most of the training examples are unlabeled and only a few are labeled.
SSL improves generalization performance by using the unlabeled examples that are often abundant. In the past decade, SSL has attracted much attention
and achieved successful results in diverse applications such as text categorization, image retrieval, and medical diagnosis.

\item \emph{Labels are implicitly known.} Multi-instance learning
(MIL) \citep{dietterich1997smi} is the most prominent example in this category. In MIL,  training examples are called {\em bags\/}, each of which
contains multiple instances. Many real-world objects can be naturally described by multiple instances. For example, an image (bag) usually contains
multiple semantic regions, and each region is an instance.
Instead of describing an object as a single instance, the multi-instance
representation can help separate different semantics within the object.
MIL has been successfully applied to diverse domains such as image
classification, text categorization, and web mining. The
relationship between multi-instance learning and semi-supervised learning has also been discussed in \citet{zhou2007relation}.

In traditional MIL, a bag is labeled positive when it contains at least one positive instance, and is labeled negative otherwise. Although the bag
labels are often available, the instance labels are only implicitly known. It is worth noting that identification of the key (or positive) instances
from the positive bags can be very useful in many real-world applications. For example, in content-based information retrieval (CBIR), the explicit
identification of regions of interest (ROI) can help the user to recognize images that he/she wants quickly (especially when the system returns a
large number of images). Similarly, to detect suspect areas in some medical and military applications, a quick scanning of a huge number of images is
required. Again, it is very desirable if ROIs can be identified.
Besides providing an accurate and efficient prediction, the identification of key instances is also useful in understanding ambiguous objects
\citep{li2012}.

\item \emph{Labels are totally unknown.} This becomes unsupervised learning
\citep{jain1988acd}, which aims at discovering the
underlying structure (or concepts/labels) of the data and grouping similar
examples together. Clustering is valuable in data
analysis, and is widely used in various domains including
information retrieval, computer version, and bioinformatics.

\item There are other kinds of weak-label learning problems. For instances, \cite{angluin1988learning}
and references therein studied noisy-tolerant
problems where the label information is noisy; \citet{sheng2008get} and references
therein considered learning from multiple annotation results by different
experts in which all the experts are imperfect; \citet{sun2010multi} and
\citet{bucak2011multi} considered weakly labeled data in the context of multi-label
learning, whereas \citet{yangIJCAI13} considered weakly labeled data in the context of multi-instance multi-label
learning \citep{zhou2012multi}.
\end{itemize}

Unlike supervised learning where the training labels are complete, weak-label learning needs to infer the integer-valued labels of the training examples, resulting in a difficult mixed-integer programming (MIP).
To solve this problem, many algorithms have been proposed, including global optimization \citep{Chapelle2008,sindhwani2006das} and convex SDP relaxations \citep{xu2005mmc,xu2005semi,de2006semi,guo2009max}. Empirical studies have demonstrated their promising performance on small data sets. Although SDP convex relaxations can reduce the training time complexity of global optimization methods from exponential to polynomial, they still cannot handle medium-sized data sets having thousands of examples. Recently, several algorithms resort to using non-convex optimization techniques (such as alternating optimization methods \citep{andrews2003svm,zhang2007mmc,li2009semi} and constrained convex-concave procedure \citep{collobert2006lst,cheung2006rfm,zhao:emm}. Although these approaches are often efficient, they can only obtain locally optimal solutions and can easily get stuck in local minima. Therefore, it is desirable to develop a scalable yet convex optimization method for learning with large-scale weakly labeled data. Moreover, unlike several scalable graph-based methods proposed for the transductive setup \citep{subramanya2009entropic,zhang2009prototype,vapnik1998statistical}, here we are more interested in inductive learning methods.

In this paper, we will focus on the binary support vector machines (SVM). Extending our preliminary works in
\citet{li2009convex,li2009tighter}, we propose a convex weakly labeled SVM (denoted \lgsvm\ (WEakly LabeLed SVM)) via a novel ``label generation" strategy. Instead of obtaining a label relation matrix via SDP, \lgsvm\ maximizes the margin
by generating the most violated label vectors iteratively, and then combines them via efficient multiple kernel
learning techniques. The whole procedure can be formulated as a convex relaxation of the original MIP problem.
Furthermore, it can be shown that the learned linear combination of label vector outer-products is in the convex hull
of the label space. Since the convex hull is the smallest convex set containing the target non-convex
set \citep{boyd2004co}, our formulation is at least as tight as the convex SDP relaxations proposed
in \citet{xu2005mmc}, \citet{de2006semi} and \citet{xu2005semi}. Moreover, \lgsvm\ involves a series of SVM subproblems, which can be
readily solved in a scalable and efficient manner via state-of-the-art SVM software such as LIBSVM \citep{fan2005wss}, SVM-\emph{perf} \citep{joachims2006tls}, LIBLINEAR \citep{hsieh2008dcd} and CVM \citep{tsang2006cvm}. Therefore,
\lgsvm\ scales much better than existing SDP approaches or even some non-convex approaches. Experiments on three common weak-label learning tasks (semi-supervised learning, multi-instance learning, and clustering) validate the
effectiveness and scalability of the proposed \lgsvm.

The rest of this paper is organized as follows. Section \ref{sec:well} briefly
introduces large margin weak-label learning. Section \ref{sec:lgsvm}
presents the proposed \lgsvm\ and analyzes its time complexity. Section
\ref{sec:implementations} presents detailed formulations on three weak-label
learning problems. Section~\ref{sec:expt} shows some comprehensive
experimental results and the last section concludes.

In the following, $\bM\pd$ (resp. $\bM\psd$) denotes that the matrix $\bM$ is
symmetric and positive definite (pd) (resp. positive semidefinite (psd)). The
transpose  of vector / matrix (in both the input and feature spaces) is
denoted by the superscript $'$, and $\0, \1 \in \R^n$ denote the zero vector
and the vector of all ones, respectively. The inequality
$\v=[v_1,\ldots,v_k]'\geq \0$ means that $v_i \geq 0$ for $i=1,\ldots,k$.
Similarly, $\bM \geq \0$ means that all elements in the matrix $\bM$ are nonnegative.

\section{Large-Margin Weak-Label Learning}
\label{sec:well}

We commence with the simpler supervised learning scenario. Given a set of labeled
examples $\mathcal{D}=\{\x_i,y_i\}_{i=1}^N$ where $\x_i \in \XX$ is the input
and  ${y}_i \in \{\pm 1\}$ is the output, we aim to find a decision
function $f: \XX \rightarrow \{\pm 1\}$ such that the following structural risk
functional is minimized:
\begin{equation}\label{eq:supervise}
\min_{f} \; \Omega(f)+ C\; \ell_{f}(\mathcal{D}).
\end{equation}
Here, $\Omega$ is a regularizer related to large margin on $f$, $\ell_f(\mathcal{D})$ is the empirical loss on
$\md$, and $C$ is a regularization parameter that trades off the empirical risk and model complexity. Both $\Omega$ and $\ell_f{(\cdot)}$ are problem-dependent. In particular, when $\ell_f{(\cdot)}$ is the hinge loss (or its variants), the obtained $f$ is a large margin classifier. It is notable that both $\Omega$ and $L_{f}(\cdot)$ are usually convex. Thus, Equation~(\ref{eq:supervise}) is a convex problem whose globally optimal solution can be efficiently obtained via various convex optimization techniques.

In weak-label learning, labels are not available on all $N$ training examples, and
so also need to be learned.
Let $\hat{\y} =[\hat{y}_1,\cdots, \hat{y}_N]'
\in \{\pm 1\}^N$ be the vector of  (known and unknown)
labels on all the training examples.
The basic idea of large-margin weak-label
learning is that the structural risk functional in Equation~(\ref{eq:supervise}) is
minimized w.r.t. both the labeling\footnote{To simplify notations, we write
$\min_{\hat{\y} \in \B}$, though indeed one only needs to minimize w.r.t. the
unknown labels in
$\hat{\y}$.}
$\hat{\y}$
and
decision function $f$. Hence, Equation~(\ref{eq:supervise}) is extended
to
\begin{eqnarray}\label{eq:weakly labeled}
\min_{\hat{\y} \in \B} \;\min_{f} \;\;\;\; \Omega(f)+C\; \ell_{f}(\{\x_i,\hat{y}_i\}_{i=1}^N),
\end{eqnarray}
where
$\B$ is a set of candidate label assignments obtained from
some domain knowledge. For example, when the positive and negative examples are
known to be approximately balanced, we can set $\B = \{\hat{\y} \; :\; -\beta \leq \sum_{i=1}^{N}\hat{y}_i \leq \beta\}$ where $\beta$ is a small constant controlling the class imbalance.

\subsection{State-of-The-Art Approaches}
\label{subsec:state-of-the-art}
As Equation~(\ref{eq:weakly labeled}) involves optimizing the integer variables $\hat{\y}$, it is no longer a convex optimization problem but a
mixed-integer program. This can easily suffer from the local minimum problem. Recently, a lot of efforts have been devoted to solve this problem.
They can be grouped into three categories. The first strategy optimizes Equation~(\ref{eq:weakly labeled}) via variants of non-convex optimization.
Examples include alternating optimization \citep{zhang2009mmc,li2009semi,andrews2003svm}, in which we alternatively optimize variable $\hat{\y}$ (or
$f$) by keeping the other variable $f$ (or $\hat{\y}$) constant;
constrained convex-concave procedure (CCCP) (also known as DC programming) \citep{horst1999dc,zhao:emm,collobert2006lst,cheung2006rfm}, in which the
non-convex objective function or constraint is decomposed as a difference of
two convex functions; local combinatorial search \citep{Joachims1999},
in which
the labels of two examples in opposite classes
are sequentially switched.
These approaches are often
computationally efficient. However, since they are based on non-convex
optimization, they may inevitably get stuck in local minima.

The second strategy obtains the globally optimal solution of
Equation~(\ref{eq:weakly labeled}) via global optimization. Examples
include branch-and-bound \citep{Chapelle2008} and deterministic
annealing \citep{sindhwani2006das}. Since they aim at obtaining the
globally optimal (instead of the locally optimal)
solution,
excellent performance
can be expected \citep{Chapelle2008}. However,
their
worst-case computational costs can scale exponentially as the data set size.
Hence, these approaches can only be applied to small data sets
with just hundreds of training examples.

The third strategy is based on convex relaxations. The original non-convex problem is
first relaxed to a convex problem, whose globally optimal solution can be efficiently obtained.
This is then rounded to recover an approximate solution of the original problem.
If the relaxation is tight, the approximate solution obtained is close to the global optimum
of the original problem and good performance can be expected. Moreover, the involved convex programming solver has a
time complexity substantially lower than that for global optimization. A prominent example of convex relaxation is the use of semidefinite programming (SDP) techniques \citep{xu2005mmc,xu2005semi,de2006semi,guo2009max}, in which a
positive semidefinite matrix is used to approximate the matrix of label outer-products. The time complexity of this SDP-based strategy is $O(N^{6.5})$ \citep{lobo1998applications,nesterov1987interior},
where $N$ is the data set size, and can be further reduced to $O(N^{4.5})$ \citep{zhang2009mmc,valizadegan1400gmm}.
However, this is still expensive for medium-sized data sets with several thousands of examples.

To summarize, existing weak-label learning approaches are not scalable or can be sensitive to initialization.
In this paper, we propose the \lgsvm\ algorithm to address these two issues.

\section{\lgsvm}
\label{sec:lgsvm}

In this section, we first introduce the SVM dual which will be used as a basic reformulation of our proposal, and then we present the general formulation of \lgsvm.
Detailed formulations on three common weak-label learning tasks will be presented in Section~\ref{sec:implementations}.

\subsection{Duals in Large Margin Classifiers}
\label{subsec:minimax-convex-relax}

In large margin classifiers, the inner
minimization problem of Equation~(\ref{eq:weakly labeled}) is often cast in
the dual form. For example, for the standard SVM without offset, we have
$\Omega=\frac{1}{2}\|\w\|^2$ and $\ell_{f}(D)$ is the summed hinge loss. The inner
minimization problem is then
\begin{eqnarray*}\label{eq:standard_svm_without_offset}
\min\limits_{\w, \bxi} && \frac{1}{2}||\w||_2^2 + C \sum_{i=1}^{N} \xi_i\\
\mbox{s.t.} && \hat{y}_i \w'\phi(\x_i) \geq 1 - \xi_i, \;\; \xi_i \geq 0, \;\; i = 1
\dots, N, \nonumber
\end{eqnarray*}
where $\phi(\x_i)$ is the feature map induced by kernel $\kappa$, and its dual is
\begin{eqnarray*}
\max\limits_{\ba} && \ba'\1-\frac{1}{2}\ba'\big(\K\odot\hat{\y}\hat{\y}'\big)\ba \\
\mbox{s.t.} && C\1 \geq \ba \geq \0,
\end{eqnarray*}
where $\ba \in \R^N$ is the dual variable, $\K \in \R^{N \times N}$ is the kernel matrix defined on the $N$
samples, and $\odot$ is the element-wise product. For more details on the duals of large margin classifiers, interested readers are referred to \citet{sch?lkopf2002learning} and \citet{cristianini2002kernel}.

In this paper, we make the following
assumption on this dual.

\begin{assm} \label{as:assm}
The dual of the inner minimization of Equation~(\ref{eq:weakly labeled}) can be written
as:
$\max_{\ba \in \A} G(\ba, \hat{\y})$,
where $\ba=[\alpha_1,\ldots,\alpha_N]'$ contains the dual variables and
\begin{itemize}
\item $\A$ is a convex set;
\item $G(\ba,\hat{\y})$ is a concave function in $\ba$ for any fixed $\hat{\y}$;
\item $g_{\y}(\ba)= -G(\ba,\y)$
is $\lambda$-strongly convex and $M$-Lipschitz. In other words, $\nabla^2 g_{\y}(\ba)- \lambda \I \psd $,  where $\I$ is the identity matrix, and
$\|g_{\y}(\ba)-g_{\y}(\bar{\ba})\| \leq M\|\ba -\bar{\ba}\|,\ \forall \y \in \B,\ \ba,\bar{\ba} \in \A$;
\item $\forall \hat{\y} \in \B$, $lb \leq \max_{\ba \in \A} G(\ba, \hat{\y})
\leq ub$, where $lb$ and $ub$ are polynomial in $N$;
\item
$G(\ba,\hat{\y})$ can be rewritten as $\bar{G}(\ba,\bM)$, where
$\bM$ is
a psd matrix, and
$\bar{G}$ is concave in $\ba$ and linear in $\bM$.
\end{itemize}
\end{assm}
With this assumption,
Equation~(\ref{eq:weakly labeled}) can be written as
\begin{equation}\label{eq:dual2}
\min_{\hat{\y} \in \B} \max_{\ba \in \A} \;\;\; G(\ba, \hat{\y}),
\end{equation}

Assume that the kernel matrix $\K$ is pd (i.e., the smallest eigenvalue $\lambda_{\min} > 0$)
and all its entries are bounded ($|K_{ij}| \leq \upsilon$ for some $\upsilon$). It is easy to see that the following SVM variants satisfy Assumption~\ref{as:assm}.

\begin{itemize}
\item Standard SVM without offset: We have
\begin{eqnarray*}
\A & = & \{\ba\;| \; C\1\geq \ba\geq \0\}, \\
G(\ba, \hat{\y}) & = & \ba'\1-\frac{1}{2}\ba'\big(\K\odot\hat{\y}\hat{\y}'\big)\ba,
\\
\nabla^2 g_{\y}(\ba) &= & \K\odot {\y}{\y}' \succeq \lambda_{\min}(\I\odot{\y}{\y}') = \lambda_{\min} \I, \\
\|g_{\y}(\ba)-g_{\y}(\bar{\ba})\| & \leq & (1+C\upsilon N)\sqrt{N}\|\ba
-\bar{\ba}\|, \\
0 & \leq & \max_{\ba \in \A} G(\ba, \hat{\y}) \leq CN, \\
\bar{G}(\ba,\bM_{\hat{\y}}) & = &
\ba'\1-\frac{1}{2}\ba'\big(\K\odot\bM_{\hat{\y}}\big)\ba,  \;\;\;\text{ where }
\bM_{\hat{\y}} = \hat{\y}\hat{\y}'.
\end{eqnarray*}
\item $\nu$-SVM \citep{sch?lkopf2002learning}: We have
\begin{eqnarray*}
\A& = & \{\ba\;|\;\ba\geq \0, \ba'\1=1\}, \\
G(\ba, \hat{\y}) & = &
-\frac{1}{2}\ba'\left((\K+\frac{1}{C}\I)\odot\hat{\y}\hat{\y}'\right)\ba, \\
\nabla^2 g_{\y}(\ba) & = & \left(\K+\frac{1}{C}\I \right)\odot{\y}{\y}' \succeq
\left(\lambda_{\min}+ \frac{1}{C} \right)\; \big(\I\odot{\y}{\y}'\big) = \left(\lambda_{\min}+
\frac{1}{C} \right) \I, \\
\|g_{\y}(\ba)-g_{\y}(\bar{\ba})\| & \leq & \left(\upsilon+\frac{1}{C} \right) N\sqrt{N}\|\ba -\bar{\ba}\|,
\\
-\frac{1}{2}\left(\upsilon + \frac{1}{C}\right) & \leq &\max_{\ba \in \A} G(\ba, \hat{\y}) \leq
0, \\
\bar{G}(\ba,\bM_{\hat{\y}}) & = & -\frac{1}{2}\ba'\left((\K+\frac{1}{C}\I)\odot\bM_{\hat{\y}}\right)\ba.
\end{eqnarray*}
\end{itemize}

\subsection{\lgsvm}

Interchanging the order of $\max_{\ba \in \A}$ and $\min_{\hat{\y} \in \B}$
in Equation~(\ref{eq:dual2}),
we obtain the proposed \lgsvm:
\begin{align}\label{eq:minimax-relax}
(\mbox{\lgsvm}) \;\;\;\; \max_{\ba \in \A} \min_{\hat{\y} \in \B}  \;\;\; G(\ba, \hat{\y}).
\end{align}
Using the minimax theorem \citep{kim08}, the optimal objective of
Equation~(\ref{eq:dual2}) upper-bounds that of Equation~(\ref{eq:minimax-relax}). Moreover,
Equation~(\ref{eq:minimax-relax}) can be transformed as
\begin{eqnarray}\label{eq:LGSVM_framework}
\!\!\! \max\limits_{\ba\in\A} \!\!\! & \Big\{ \max_\th \!\! & \! \th \\
&\mbox{s.t.} \!\!\! \! & \!\! G(\ba,\hat{\y}_t) \geq \th,\;\forall\hat{\y}_t \in\B \Big\}, \nonumber
\end{eqnarray}
from which we obtain the following Proposition.

\begin{prop}
The objective of \lgsvm\ can be rewritten as the following optimization problem:
\begin{eqnarray}\label{eq:LGSVM_framework_v2}
\min_{\bmu\in \M}\max_{\ba\in\A}
\sum_{t:\hat{\y}_t \in \B}\mu_t G(\ba,\hat{\y}_t),
\end{eqnarray}
where $\bmu$ is the vector of $\mu_t$'s, $\M$ is the simplex $\{\bmu \; | \; \sum_t \mu_t=1, \mu_t\geq 0\}$, and
$\hat{\y}_t\in \B$.
\end{prop}
\begin{proof}
For the inner optimization in Equation~(\ref{eq:LGSVM_framework}), let $\mu_t \geq 0$ be the dual variable for each constraint. Its Lagrangian can be obtained as
\[ \theta+\sum_{t:\hat{\y}_t \in\B}
\mu_t \Big(G(\ba,\hat{\y}_t) - \theta \Big). \]
Setting the derivative w.r.t. $\theta$ to zero, we have $\sum_{t}\mu_t = 1$.  We can then replace the inner optimization subproblem with its dual and Equation~(\ref{eq:LGSVM_framework}) becomes:
\[ \max_{\ba\in\A} \min_{\bmu\in \M}
\sum_{t:\hat{\y}_t \in\B}\mu_t G(\ba,\hat{\y}_t) =
\min_{\bmu\in \M}\max_{\ba\in\A}
\sum_{t:\hat{\y}_t \in\B}\mu_t G(\ba,\hat{\y}_t). \]
Here, we use the fact that the objective function is convex in $\bmu$ and
concave in $\ba$.
\end{proof}

Recall that $G(\ba, \hat{\y})$ is concave in $\ba$. Thus, the constraints in Equation~(\ref{eq:LGSVM_framework}) are convex. It is evident that the objective in Equation~(\ref{eq:LGSVM_framework}) is linear in both $\ba$ and $\th$. Therefore, Equation~(\ref{eq:LGSVM_framework}) is a \emph{convex} problem. In other words, \lgsvm\ is a convex relaxation of
Equation~(\ref{eq:weakly labeled}).

\subsection{Tighter than SDP Relaxations}
\label{subsec:tight}

In this section, we compare our minimax relaxation with SDP relaxations. It is
notable that the SVM without offset is always employed by previous SDP
relaxations \citep{xu2005mmc,xu2005semi,de2006semi}.

Recall the symbols in Section~\ref{subsec:minimax-convex-relax}. Define
\begin{equation*}
\YY_0 = \big\{\bM \; | \; \bM = \bM_{\hat{\y}}, \;\; \hat{\y} \in \B \big\}.
\end{equation*}
The original mixed-integer program in Equation~(\ref{eq:dual2}) is the same as
\begin{equation} \label{eq:exactly_well}
\min_{\bM \in \YY_0}\max_{\ba\in\A} \; \bar{G}(\ba,\bM).
\end{equation}
Define $\YY_1 = \big\{\bM \; | \; \bM = \sum_{t:\hat{\y}_t \in \B}\mu_t \bM_{\hat{\y}_t}, \;\; \bmu \in \M  \big\}$. Our minimax relaxation in
Equation~(\ref{eq:LGSVM_framework_v2}) can be written as
\begin{eqnarray}
\min_{\bmu\in \M} \max_{\ba\in\A} \sum_{t:\hat{\y}_t \in \B}\mu_t \bar{G}(\ba,\bM_{\hat{\y}_t}) & = & \min_{\bmu\in \M} \max_{\ba\in\A}
\bar{G}\left(\ba,\sum_{t:\hat{\y}_t \in \B}\mu_t \bM_{\hat{\y}_t}\right) \nonumber\\
&=& \min_{\bM \in \YY_1}\max_{\ba\in\A} \; \bar{G}(\ba,\bM). \label{eq:well_minimax}
\end{eqnarray}
On the other hand, the SDP relaxations
in \citet{xu2005mmc,xu2005semi} and \citet{de2006semi} are
of the form
\begin{equation*}
\min_{\bM \in \YY_2}\max_{\ba\in\A} \; \bar{G}(\ba,\bM),
\end{equation*}
where
$\YY_2 = \big\{\bM \; | \; \bM \psd , \bM \in \M_{\B} \big\}$,
and
$\M_{\B}$ is
a convex set related to $\B$.
For example, in the context of clustering,
\citet{xu2005mmc} used
$\B=\{\hat{\y}|-\beta\leq
\1'\hat{\y}\leq \beta\}$, where $\beta$ is a parameter controlling the class
imbalance, and $\M_{\B}$ is defined as
\begin{eqnarray*}
\M_{\B}^{\text{clustering}} = & \Big\{\bM
=[m_{ij}]
\;| & -1\leq m_{ij}\leq 1; m_{ii}=1, m_{ij}=m_{ji}, \nonumber\\
&&m_{ik}\geq m_{ij}+m_{jk}-1, m_{jk} \geq -m_{ij}-m_{ik}-1, \nonumber\\
&&-\beta \leq \sum_{i=1}^{N}m_{ij} \leq \beta, \;\; \forall i,j,k=1,\dots, N \Big\}.
\end{eqnarray*}
It is easy to verify that $\YY_{0} \subseteq \YY_2$ and $\YY_2$ is convex. Similarly,
in semi-supervised learning, \citet{xu2005semi} and
\citet{de2006semi}
defined $\M_{\B}$ as
a subset\footnote{For a more precise
definition, interested readers are referred to
\citet{xu2005semi} and \citet{de2006semi}.}
of $\M_{\B}^{\text{clustering}}$. Again,
$\YY_{0} \subseteq \YY_2$ and $\YY_2$ is convex.

\begin{them}\label{them:relaxation} The relaxation of \lgsvm\ is at least as tight as
the SDP relaxations in
 \citet{xu2005mmc,xu2005semi} and \citet{de2006semi}.
\end{them}
\begin{proof}
Note that $\YY_1$ is the convex hull of $\YY_0$, that is, the smallest convex set containing $\YY_0$ \citep{boyd2004co}. Therefore, Equation~(\ref{eq:well_minimax}) gives the tightest convex relaxation of
Equation~(\ref{eq:exactly_well}), that is, $\YY_1 \subseteq \YY_2$. In other words, our
relaxation is at least as tight as SDP relaxations.
\end{proof}

\subsection{Cutting Plane Algorithm by Label Generation}
\label{subsec:cp-algo}

It appears that existing convex optimization techniques can be readily used to solve the convex problem in Equation~(\ref{eq:LGSVM_framework_v2}), or
equivalently Equation~(\ref{eq:LGSVM_framework}). However, note that there can be an exponential number of
constraints in Equation~(\ref{eq:LGSVM_framework}),
and so a direct optimization is computationally intractable. Fortunately, typically not all these constraints
are active at optimality, and including only a subset of them can lead to a very good approximation of the original optimization problem. Therefore,
we can apply the cutting plane method \citep{kelley1960cpm}.

The cutting plane algorithm is described in Algorithm~\ref{alg:LGSVM}. First, we initialize a label vector $\hat{\y}$ and the working set $\C$ to $\{\hat{\y}\}$, and obtain $\ba$ from

\begin{eqnarray}\label{eq:LGSVM_framework_v22}
\min_{\bmu\in \M}\max_{\ba\in\A}
\sum_{t : \hat{\y}_t \in \C}\mu_t G(\ba,\hat{\y}_t)
\end{eqnarray}
via standard supervised
learning
methods. Then, a violated label vector $\hat{\y}$ in Equation~(\ref{eq:LGSVM_framework})
is \emph{generated} and added to $\C$.
The process is repeated until the termination criterion is met.
Since
the size of the working set $\C$ is often much smaller than that of $\B$,
one can use existing convex optimization techniques
to obtain $\ba$ from
Equation~(\ref{eq:LGSVM_framework_v22}).

For the non-convex optimization methods reviewed in Section~\ref{subsec:state-of-the-art},
a new label assignment for the unlabeled data is also generated in each iteration.
However, they are very different from our proposal. First, those algorithms
do not take the previous label assignments into account, while, as will be seen in Section 4.1.2, our \lgsvm\ aims to learn a combination of previous label assignments. Moreover, they update the label assignment
to approach a locally optimal solution, while our \lgsvm\ aims to obtain a tight convex relaxation solution.

\begin{algorithm}[t]
\caption{Cutting plane algorithm for \lgsvm.}
\begin{algorithmic}[1]
\STATE Initialize $\hat{\y}$ and $\C = \emptyset$.
\REPEAT
\STATE Update $\C \leftarrow \{\hat{\y}\} \bigcup \C$.
\STATE Obtain the optimal $\ba$ from Equation~(\ref{eq:LGSVM_framework_v22}).
\STATE Generate a violated $\hat{\y}$.
\UNTIL{$G(\ba,\hat{\y}) > \min_{\y \in
\C} G(\ba,{\y})-\eps$
(where $\eps$ is a small constant)} or the decrease of objective value is smaller than a threshold.
\end{algorithmic}
\label{alg:LGSVM}
\end{algorithm}

\subsection{Computational Complexity}
\label{subsec:time-analysis}

The key to analyzing the running time of Algorithm~\ref{alg:LGSVM} is its convergence rate, and we have the following Theorem.

\begin{them}\label{them:objective_gain}
Let $p^{(t)}$ be the optimal objective value of Equation~(\ref{eq:LGSVM_framework_v22}) at the $t$-th iteration. Then,
\begin{equation}\label{eq:objective_gain}
p^{(t+1)} \leq p^{(t)}-\eta,
\end{equation}
where
$\eta = \Big( \frac{-c+\sqrt{c^2+4\epsilon}}{2} \Big)^2$, and
$c = M\sqrt{2/\lambda}$.
\end{them}
Proof is in Appendix~\ref{app:proof of theorem}.
From Theorem~\ref{them:objective_gain}, we can obtain the following convergence rate.
\begin{prop}\label{prop:step}
Algorithm
\ref{alg:LGSVM} converges in no more
than
$\frac{p^{(1)}-p^{*}}{\eta}$ iterations, where
$p^{*}$ is the optimal objective value of \lgsvm.
\end{prop}
According to Assumption 1, we have $p^{*} = \min_{\hat{\y} \in \B} \max_{\ba
\in \A} G(\ba, \hat{\y}) \geq lb$ and $p^{(1)} = \max_{\ba \in \A} G(\ba,
\hat{\y}) \leq ub$. Moreover, recall that $lb$ and $ub$ are polynomial in $N$.
Thus, Proposition~\ref{prop:step} shows that with the use of the cutting plane
algorithm, the number of active constraints only scales polynomially in $N$. In particular, as discussed in
Section~\ref{subsec:minimax-convex-relax}, for the $\nu$-SVM, $lb =
-\frac{1}{2}(\upsilon+\frac{1}{C})$ and $ub = 0$, both of which are unrelated to $N$. Thus, the number of active constraints only scales as $O(1)$.

Proposition~\ref{prop:step} can be further refined
by taking the search effort of a violated label into account. The proof is similar to that of Theorem~\ref{them:objective_gain}.

\begin{prop}\label{prop:step-plus}
Let $\epsilon_r \geq \epsilon$, $\forall r =1,2,\ldots$, be the magnitude of the
violation of a violated label in the $r$-th iteration, that is, $\epsilon_r = \min_{\y
\in \C_r} G(\ba,{\y})-G(\ba,\hat{\y}^r)$, where $\C_r$ and $\hat{\y}^r$ denote the
set of violated labels and the violated label obtained in the $r$-th iteration, respectively. Let $\eta_r = \Big( \frac{-c+\sqrt{c^2+4\epsilon_r}}{2} \Big)^2$. Then, Algorithm \ref{alg:LGSVM} converges in no more
than $R$ iterations where $\sum_{r=1}^{R} \eta_r \geq p^{(1)}-p^{*}$.
\end{prop}
Hence, the more effort is spent on finding a violated label, the faster is the
convergence. This represents a trade-off between the convergence rate and cost in each iteration.

We will show in Section~\ref{sec:implementations} that step 4 of Algorithm~\ref{alg:LGSVM} can
be addressed via multiple kernel learning techniques which only involve a series of SVM subproblems that can be solved efficiently by state-of-the-art SVM software such as LIBSVM \citep{fan2005wss} and LIBLINEAR \citep{hsieh2008dcd}, while step 5 can be efficiently addressed by sorting. Therefore, the total time complexity of \lgsvm\
scales as the existing SVM solvers, and is significantly faster than SDP relaxations.

\section{Three Weak-Label Learning Problems}
\label{sec:implementations}

In this section, we present the detailed formulations of \lgsvm\ on three
common weak-label learning tasks, namely, semi-supervised learning (Section
\ref{subsec:ssl}), multi-instance learning
(Section \ref{subsec:mil}), and clustering (Section \ref{subsec:clustering}).

\subsection{Semi-Supervised Learning}
\label{subsec:ssl}

In semi-supervised learning, not all
the training labels
are known. Let
$\md_{\mL}=\{\x_i,y_i\}_{i=1}^{l}$ and $\md_{\mU}=\{\x_j\}_{j=l+1}^{N}$ be the
sets of labeled and unlabeled examples, respectively,
and ${\mL} = \{1,\ldots, l\}$ (resp. $\mU = \{l+1,\ldots, N\}$) be the index set of
the labeled (resp. unlabeled) examples.
In semi-supervised learning, unlabeled data are
typically much more abundant than labeled data, that is, $N-l \gg l$. Hence, one
can obtain a trivially ``optimal'' solution with infinite margin by assigning
all the unlabeled examples to the same label. To prevent such a useless
solution, \citet{Joachims1999} introduced the balance constraint
\[
\frac{\1'\hat{\y}_{\mU}}{N-l} = \frac{\1'{\y}_{\mL}}{l},
\]
where $\hat{\y}=[\hat{y}_1,\cdots, \hat{y}_N]'$ is the vector of learned labels on both
labeled and unlabeled examples, $\y_{\mL} = [y_1,\dots,y_l]'$, and
$\hat{\y}_{\mU} = [\hat{y}_{l+1},\dots,\hat{y}_{N}]'$. Let $\Omega = \frac{1}{2}\|\w\|^2$ and
$\ell_{f}(\md)$ be the sum of hinge loss values on both labeled and unlabeled data,
Equation~(\ref{eq:weakly labeled}) leads to
\begin{eqnarray}
\min_{\hat{\y} \in \B}\min\limits_{\w, \bxi} && \frac{1}{2}||\w||_2^2 + C_1 \sum_{i=1}^{l} \xi_i + C_2 \sum_{i=l+1}^{N} \xi_i
\nonumber \\
\mbox{s.t.} && \hat{y}_i \w'\phi(\x_i) \geq 1 - \xi_i, \;\; \xi_i \geq 0,\;\; i = 1
\dots, N, \nonumber
\end{eqnarray}
where $\B = \{\hat{\y} \;|\; \hat{\y} = [\hat{\y}_{\mL}; \hat{\y}_{\mU}],
\hat{\y}_{\mL} = \y_{\mL}, \hat{\y}_{\mU} \in \{\pm 1\}^{N-l};
\frac{\1'\hat{\y}_{\mU}}{N-l} = \frac{\1'{\y}_{\mL}}{l} \}$, and $C_1, C_2$
trade off model complexity and empirical losses on the labeled and unlabeled
data, respectively. The inner minimization problem can be rewritten in its
dual, as:
\begin{eqnarray} \label{eq:ssl_dual2}
&\min_{\hat{\y} \in \B} \max_{\ba\in\A} &  G(\ba,\hat{\y}) := \1'\ba - \frac{1}{2}\ba'\Big(\K\odot \hat{\y}\hat{\y}' \Big) \ba,
\end{eqnarray}
where $\ba=[\a_1,\dots,\a_N]'$ is the vector of dual variables, and
$\A=\{\ba \; \big| \; C_1 \geq \alpha_i \geq 0, C_2 \geq \alpha_j \geq 0, i \in
\mL, j \in \mU\}$.

Using Proposition 1, we have
\begin{align}
\min_{\bmu\in \M} \max_{\ba\in\A} \;\; \1'\ba-
\frac{1}{2}\ba'\Big(\sum_{t:\hat{\y}_t\in\B}\mu_t\K\odot
\hat{\y}_t\hat{\y}_t' \Big) \ba, \label{eq:ssl_framework}
\end{align}
which is a convex relaxation of Equation~(\ref{eq:ssl_dual2}). Note that $G(\ba,\hat{\y})$ can be rewritten as
$\bar{G}(\ba,\bM_{\y})=\1'\ba - \frac{1}{2}\ba'\Big(\K\odot \bM_{\y}
\Big) \ba$, where $\bar{G}$ is concave in $\ba$ and linear in $\bM_{\y}$. Hence, according to Theorem 1, \lgsvm\
is at least as tight as the SDP relaxations in \citet{xu2005semi} and \citet{de2006semi}.

Notice the similarity with standard SVM, which involves a single kernel matrix
$\K\odot \hat{\y}\hat{\y}'$. Hence, Equation~(\ref{eq:ssl_framework}) can be regarded as
\emph{multiple kernel learning} (MKL) \citep{lanckriet2004lkm}, where the target
kernel matrix is a convex combination of $|\B|$ base kernel matrices $\{\K\odot
\hat{\y}_t \hat{\y}_t'\}_{t:\hat{\y}_t \in\B}$, each of which is constructed from a feasible label vector $\hat{\y}_t \in \B$.

\subsubsection{Algorithm}

 From Section~\ref{sec:lgsvm}, the cutting plane algorithm is used to
solve Equation~(\ref{eq:ssl_framework}). There are two important issues
that have to be addressed in the use of cutting plane algorithms. First, how to efficiently solve the MKL optimization problem? Second, how to efficiently find a violated $\hat{\y}$? These will be addressed in Sections~\ref{sec:mlkl} and \ref{sec:violate}, respectively.

\subsubsection{Multiple Label-Kernel Learning}
\label{sec:mlkl}

In recent years, a lot of efforts have been devoted on efficient MKL approaches. \citet{lanckriet2004lkm} first proposed the use of quadratically
constrained quadratic programming (QCQP) in MKL. \citet{bach2004mkl} showed that an approximate solution can be efficiently obtained by using
sequential minimization optimization (SMO) \citep{platt-99b}. Recently, \citet{sonnenburg2006lsm} proposed a semi-infinite linear programming (SILP)
formulation which allows MKL to be iteratively solved with standard SVM solver and linear programming. \citet{rakotomamonjy-08} proposed a weighted
2-norm regularization with additional constraints on the kernel weights to encourage a sparse kernel combination. \citet{xu-09} proposed the use of
the extended level method to improve its convergence, which is further refined by the MKLGL algorithm \citep{xusimple}. Extension to nonlinear MKL
combinations is also studied recently in \citet{kloft2009efficient}.

Unlike standard MKL problems which try to find the optimal kernel function/matrix for a given set of labels, here, we have to find the optimal label kernel matrix. In this paper, we use an adaptation of the MKLGL algorithm \citep{xusimple} to solve this multiple label-kernel learning (MLKL) problem. More specifically, suppose that the current working
set is $\C =\{\hat{\y}_1,\dots,\hat{\y}_T\}$. Note that the feature map
corresponding to the base kernel matrix $\K\odot \hat{\y}_t \hat{\y}_t'$ is
$\x_i \mapsto \hat{y}_{ti}\phi(\x_i)$. The MKL problem in Equation~(\ref{eq:ssl_framework}) thus
corresponds to the following primal optimization problem:
\begin{eqnarray}
& \min\limits_{\bmu\in\M,\bW=[\w_1,\dots,\w_T],\bxi} &
\frac{1}{2}\sum_{t=1}^{T}\frac{1}{\mu_{t}} ||\w_t||^{2} +
C_1 \sum_{i=1}^l\xi_i + C_2 \sum_{i=l+1}^{N} \xi_i \label{eq:ssl_solve_beta} \\
&\mbox{s.t.} & \sum_{t=1}^{T}\hat{y}_{ti}\w_t'\phi(\x_i) \geq 1 -
\xi_i, \;\; \xi_i \geq 0, \;\; i = 1,\dots, N. \nonumber
\end{eqnarray}
It is easy to verify that its dual can be written as
\begin{eqnarray*}
\min\limits_{\bmu \in \M} \max\limits_{\ba \in \A} & \1'\ba-\frac{1}{2} \ba'\Big(\sum_{t=1}^{T}\mu_t
\K \odot \hat{\y}_t \hat{\y}_t' \Big)\ba,
\end{eqnarray*}
which is the same as Equation~(\ref{eq:ssl_framework}). Following MKLGL, we can solve
Equation~(\ref{eq:ssl_framework}) (or, equivalently, Equation~(\ref{eq:ssl_solve_beta})) by
iterating
the following two steps until convergence.
\begin{enumerate}
\item
Fix the mixing coefficients $\bmu$ of the base
kernel matrices and solve Equation~(\ref{eq:ssl_solve_beta}). By setting
$\tilde{\w} = [\frac{\w_1}{\sqrt{\mu_1}},\ldots,\frac{\w_T}{\sqrt{\mu_T}}]'$, $\tilde{\x}_i =
[\sqrt{\mu_1}\phi(\x_i),\sqrt{\mu_2}\hat{y}_{1i}\hat{y}_{2i}\phi(\x_i),
\dots,\sqrt{\mu_T}\hat{y}_{1i}\hat{y}_{Ti}\phi(\x_i)]'$ and $\tilde{\y} =
\hat{\y}_1$, Equation~(\ref{eq:ssl_solve_beta}) can be rewritten as
\begin{eqnarray*}
& \min\limits_{\tilde{\w},\bxi} &
\frac{1}{2}||\tilde{\w}||^{2} + C_1 \sum_{i=1}^l\xi_i + C_2 \sum_{i=l+1}^{N} \xi_i \\
&\mbox{s.t.} & \tilde{y}_i \tilde{\w}'\tilde{\x}_i \geq 1 -
\xi_i, \;\; \xi_i \geq 0, \;\; i = 1,\dots, N, \nonumber
\end{eqnarray*}
which is similar to the primal
of the standard
SVM and can be efficiently handled by state-of-the-art SVM solvers.

\item Fix $\w_t$'s and update $\bmu$ in closed-form, as
\begin{equation*}
\mu_t = \frac{\|\w_t\|}{\sum_{t'=1}^{T} \|\w_{t'}\|}, \;\; t = 1,\dots, T.
\end{equation*}
\end{enumerate}

In our experiments, this always converges in fewer than $100$ iterations. With the use of warm-start, even faster convergence can be expected.

\subsubsection{Finding a Violated Label Assignment}
\label{sec:violate}

The following optimization problem corresponds to finding the most violated $\hat{\y}$
\begin{align}\label{eq:most_violate_y}
\min\limits_{\hat{\y}\in \B} & \;\;\; G(\ba,\hat{\y}) = \1'\ba - \frac{1}{2}\ba'\Big( \K \odot \hat{\y}
\hat{\y}' \Big)\ba.
\end{align}
The first term in the objective does not relate to $\hat{\y}$, so
Equation~(\ref{eq:most_violate_y}) is rewritten as
\begin{align*}
\max\limits_{\hat{\y}\in \B} & \;\; \frac{1}{2}\ba'\Big( \K \odot \hat{\y}
\hat{\y}' \Big)\ba.
\end{align*}
However, this is a concave QP and cannot be solved efficiently. Note that while
the use of the most violated constraint may lead to faster convergence, the
cutting plane algorithm only requires the addition of a violated constraint at
each iteration \citep{kelley1960cpm,tsochantaridis2006large}. Hence, we propose in the following a simple and efficient
method for finding a violated label assignment.

Consider the following equivalent problem:
\begin{align}\label{eq:mostviolate}
\max\limits_{\hat{\y}\in \B} \hat{\y}'\bH\hat{\y},
\end{align}
where $\bH = \K \odot (\ba \ba')$ is a psd matrix. Let $\bar{\y} \in \C$ be the
following suboptimal solution of Equation~(\ref{eq:mostviolate})
\[
\bar{\y} = \mathop{\arg\max}\nolimits_{\hat{\y}\in \C} \hat{\y}'\bH\hat{\y}.
\]

Consider an optimal solution of the following optimization problem
\begin{align}\label{eq:a_violated_y}
{\y}^{*} = \argmax\limits_{\hat{\y}\in \B} \hat{\y}'\bH\bar{\y}.
\end{align}
We have the following proposition.
\begin{prop}\label{prop:violate}
${\y}^{*}$ is a violated label assignment if $\bar{\y}'\bH{\y}^{*} \neq \bar{\y}'\bH\bar{\y}$.
\end{prop}
\begin{proof}
From $\hat{\y}'\bH{\y}^{*} \neq \bar{\y}'\bH\bar{\y}$, we have ${\y}^{*} \neq \bar{\y}$. Suppose that
$({\y}^{*})'\bH {\y}^{*} \leq \bar{\y}'\bH\bar{\y}$, then $({\y}^{*})'\bH
{\y}^{*} + \bar{\y}'\bH\bar{\y}-2({\y}^{*})'\bH \bar{\y} \leq
2\bar{\y}'\bH\bar{\y} - 2({\y}^{*})'\bH \bar{\y}< 0$ which contradicts with $({\y}^{*})'\bH {\y}^{*} + \bar{\y}'\bH\bar{\y} - 2({\y}^{*})'\bH \bar{\y} = ({\y}^{*} - \bar{\y})'\bH ({\y}^{*}- \bar{\y}) \geq 0$. So, $({\y}^{*})'\bH {\y}^{*} > \bar{\y}'\bH\bar{\y}$ which indicates ${\y}^{*}$ is a violated label assignment.
\end{proof}

As for solving Equation~(\ref{eq:a_violated_y}), it is a integer linear program for
$\hat{\y}$. We can rewrite this as
\begin{eqnarray}\label{eq:toy}
& \max\limits_{\hat{\y}} & \br'\hat{\y} = \br_{\mL}'\hat{\y}_{\mL} +
\br_{\mU}'\hat{\y}_{\mU} \\
& \mbox{s.t.} & \hat{\y}_{\mL} = \y_{\mL}, \hat{\y}_{\mU} \in \{\pm 1\}^{N-l},
\frac{\1'\hat{\y}_{\mU}}{N-l} = \frac{\1'{\y}_{\mL}}{l}, \nonumber
\end{eqnarray}
where $\br = \bH\bar{\y}$.
Since $\hat{\y}_{\mL}$ is constant, we have the following proposition.
\begin{prop}\label{prop:ordering}
At optimality, $\hat{y}_i \geq \hat{y}_j$ if $r_i > r_j$, $i,j\in
\mU$.
\end{prop}
\begin{proof} Assume, to the contrary, that the optimal $\hat{\y}$ does not have the same sorted order as $\br$. Then, there are two label vectors $\hat{y}_i$ and $\hat{y}_j$, with $r_i > r_j$ but $\hat{y}_i < \hat{y}_j$. Then $r_i \hat{y}_i + r_j \hat{y}_j < r_i \hat{y}_j + r_j \hat{y}_i$ as $(r_i - r_j)(\hat{y}_i - \hat{y}_j) < 0$. Thus,
$\hat{\y}$ is not optimal, a contradiction.
\end{proof}

Thus, with Proposition~\ref{prop:ordering}, we can solve Equation~(\ref{eq:toy}) by first
sorting in ascending order.
The label
assignment of $\hat{y}_i$'s aligns with the sorted values of $r_i$'s for $i \in \mU$. To satisfy the balance constraint $ \frac{\1'\hat{\y}_{\mU}}{N-l} = \frac{\1'{\y}_{\mL}}{l}$, the first $ \left\lceil \frac{1}{2}\left((N-l)(1-\frac{1}{l}1'{\y}_{\mL}) \right) \right\rceil$ of $\hat{y}_i$'s are assigned
$-1$, while the last $(N-l)-\left\lceil \frac{1}{2}\left((N-l)(1-\frac{1}{l}1'{\y}_{\mL}) \right) \right\rceil$
of them are assigned 1. Therefore, the label assignment in problem Equation~(\ref{eq:toy}) can be determined exactly and efficiently by sorting.

To find a violated label, we first get the $\bar{\y} \in \C$, which takes $O(N^2)$ (resp. $O(N)$) time when a nonlinear (resp. linear)\footnote{When the linear kernel is used, Equation~(\ref{eq:mostviolate}) can be rewritten as $\max\limits_{\hat{\y}\in \C} \;\; (\ba\odot \hat{\y})'\X'\X(\ba\odot \hat{\y})$,
where $\X=[\x_1,\dots,\x_N]$. Hence, one can first compute $\o=\X(\ba \odot \hat{\y})$ and
then compute $\o'\o$. This takes a total of $O(N)$ time. A similar trick can be used in checking if
$\y^{*}$ is a violated label assignment.} kernel is used; next we obtain the $\y^{*}$ in Equation~(\ref{eq:a_violated_y}), which takes $O(N\log N)$ time; and finally check if $\y^{*}$ is a violated label assignment using Proposition~\ref{prop:violate}, which takes $O(N^2)$ (resp. $O(N)$) time for a nonlinear (resp. linear)
kernel. In total, this takes $O(N^2)$ (resp. $O(N\log N)$) time for nonlinear (resp. linear) kernel.
Therefore, our proposal is computationally efficient.

Finally, after finishing the training process, we use $f(\x) = \sum_{t=1}^{T}\w_t'\phi(\x)$ as the prediction function. Algorithm \ref{alg:SSL} summarizes the pseudocode of \lgsvm\ for semi-supervised learning.

\begin{algorithm}[htbp]
\caption{\lgsvm\ for semi-supervised learning.}
\begin{algorithmic}[1]
\STATE Initialize $\hat{\y}$ and $\C = \emptyset$.
\REPEAT
\STATE Update $\C \leftarrow \{{\y}^*\} \bigcup \C$.
\STATE Obtain the optimal $\{\bmu,\bW\}$ or $\ba$ from Equation~(\ref{eq:ssl_solve_beta}).
\STATE Find the optimal solution ${\y}^{*}$ of Equation~(\ref{eq:a_violated_y}).
\UNTIL{ $G(\ba,\y^{*}) > \min_{\y \in \C}G(\ba,\y) - \eps$ or the decrease of objective value is smaller than a threshold.}
\STATE Output $f(\x) = \sum_{t=1}^{T}\w_t'\phi(\x)$ as our prediction function.
\end{algorithmic}
\label{alg:SSL}
\end{algorithm}

\subsection{Multi-Instance Learning}
\label{subsec:mil}

In this section, we consider the second weakly labeled learning problem, namely, multi-instance learning (MIL), where examples are bags containing multiple instances. More formally, we have a data set $\md = \{\bB_i,y_i\}_{i=1}^{m}$, where $\bB_i = \{\x_{i,1},\dots,\x_{i,m_i}\}$ is the input bag, $y_i \in \{\pm 1\}$ is the output and $m$ is the number of bags. Without loss of generality, we assume that the positive bags are ordered before the negative bags, that is, $y_i = 1$ for all $1 \leq i \leq p$ and $-1$ otherwise. Here, $p$ and $m-p$ are the numbers of positive and negative bags, respectively. In traditional MIL, a bag is labeled positive if it contains at least one key (or positive) instance, and negative otherwise. Thus, we only have the bag labels available, while the instance labels are only implicitly known.

Identification of the key instances from positive bags can be very useful in CBIR. Specifically, in CBIR, the whole image (bag) can be represented by multiple semantic regions (instances). Explicit identification of the regions
of interest (ROIs) can help the user in recognizing images he/she wants quickly especially when the system returns a large amount of images. Consequently, the problem of determining whether a region is ROI can be posed as finding the key instances in MIL.

Traditional MIL implies that the label of a bag is determined by its most representative key instance,
that is, $f(\bB_i) = \max\{f(\x_{i,1}),\cdots, f(\x_{i,m_i})\}$. Let $\Omega=\frac{1}{2}\|\w\|_2^{2}$ and $\ell_{f}(\md)$ be the sum of hinge losses on the bags, Equation~(\ref{eq:weakly labeled}) then leads to the MI-SVM proposed in \citet{andrews2003svm}:
\begin{eqnarray}
\min\limits_{\w, \bxi} && \frac{1}{2}||\w||_2^2 + C_1 \sum_{i=1}^{p} \xi_i + C_2 \sum_{i=p+1}^{m} \xi_i
\label{eq:MIL_primal}\\
\mbox{s.t.} && {y}_i \max_{1 \leq j \leq m_i}\w'\phi(\x_{i,j}) \geq 1 - \xi_i, \;\; \xi_i \geq 0, \;\; i = 1,
\dots, m. \nonumber
\end{eqnarray}
Here, $C_1$ and $C_2$ trade off the model complexity and empirical losses on the positive and negative bags, respectively.

For a positive bag $\bB_i$, we use  the binary vector $\dd_i = [d_{i,1},\cdots, d_{i,m_i}]' \in \{0,1\}^{m_i}$ to
indicate which instance in $\bB_i$ is its key instance. Following the traditional MIL setup, we
assume that each positive bag has only one key instance,\footnote{Sometimes, one can allow for more than one key instances in a positive bag \citep{wang2008app,xu2004lra,Zhou:Zhang2007nips06,zhou2012multi}. The proposed method can be extended to this case by setting $\sum_{j=1}^{m_i}d_{i,j} = v$, where $v$ is the known
number of key instances.} and so $\sum_{j=1}^{m_i}d_{i,j} = 1$. In the following,
let
$\ss= [\dd_1; \dots; \dd_p]$, and $\D$ be its domain. Moreover, note that $\max_{1\leq j \leq m_i} \w'\phi(\x_{i,j})$ in Equation~(\ref{eq:MIL_primal}) can be written as $\max_{\dd_i} \sum_{j=1}^{m_i}d_{i,j}\w'\phi(\x_{i,j})$.

For a negative bag $\bB_i$, all its instances are negative and the
corresponding constraint Equation~(\ref{eq:MIL_primal}) can be replaced by
$-\w'\phi(\x_{i,j}) \geq 1 - \xi_i$ for every instance $\x_{i,j}$ in $\bB_i$.
Moreover, we relax the problem by allowing the slack variables $\xi_i$'s to be
different for different instances in $\bB_i$. This leads to a set of slack
variables $\{\xi_{s(i,j)}\}_{i=p+1,\dots,m;j=1,\dots,m_i}$, where the indexing
function $s(i,j)=J_{i-1}-J_{p}+j+p$ ranges from $p+1$ to $q =N-J_p+p$ and $J_i =
\sum_{t=1}^{i} m_t$ ($J_0$ is set to 0).
Combining all these together, Equation~(\ref{eq:MIL_primal}) can be rewritten as
\begin{eqnarray*}
& \min\limits_{\ss\in\Delta}\min\limits_{\w, \bxi} &
\frac{1}{2}||\w||_2^2 + C_1 \sum_{i=1}^{p}\xi_i  + C_2 \sum_{i=p+1}^{m}\sum_{j=1}^{m_i}\xi_{s(i,j)} \\
& \mbox{s.t.} &  \sum_{j=1}^{m_i}\w' d_{i,j}
\phi(\x_{i,j}) \geq 1 - \xi_i, \;\; \xi_i \geq 0, \;\; i  = 1,\dots, p, \nonumber\\
&&  - \w'\phi(\x_{i,j}) \geq 1 - \xi_{s(i,j)}, \;\; \xi_{s(i,j)} \geq 0, \;\; i = p+1,\dots,m, j=1,\dots,m_i. \nonumber
\end{eqnarray*}
The inner minimization problem is usually written in its dual, as
\begin{eqnarray}
\max_{\ba \in \A} &  G(\ba,\ss) = \1'\ba - \frac{1}{2} (\ba \odot
\hat{\y})'\Big (\K^{\ss} \Big)(\ba \odot \hat{\y}),
\label{eq:dual-mil}
\end{eqnarray}
where $\ba=[\a_1,\dots,\a_{q}]' \in \R^{q}$ is the vector of dual variables, $\A = \{\ba \; | \;
C_1 \geq \a_i \geq 0, C_2 \geq \a_j \geq 0, i = 1,\dots,p; j = p+1, \dots,
q\}$, $\hat{\y} = [\1_p,-\1_{q-p}] \in \R^q$, $\K^{\ss} \in \R^ {q \times q}$ is the kernel matrix where
$\K^{\ss}_{ij} = (\bpsi^{\ss}_i)'(\bpsi^{\ss}_j)$ with
\begin{equation}\label{eq:phi}
\bpsi^{\ss}_i = \left\{ \begin{array}{ll}
\sum_{j=1}^{m_i}d_{i,j}\phi(\x_{i,j}) & i=1,\dots,p,  \\
\phi(\x_{s(i,j)}) & i=p+1,\dots,m; j=1,\dots,m_i.
\end{array} \right.
\end{equation}
Thus, Equation~(\ref{eq:dual-mil}) is a mixed-integer  programming problem. With Proposition 1, we
have
\begin{align}
 \min_{\bmu\in \M} \max_{\ba\in\A} \;\;   \1'\ba - \frac{1}{2} (\ba \odot
\hat{\y})'\sum_{t:\ss_t \in \D}\Big (\mu_t \K^{\ss_t} \Big)(\ba \odot \hat{\y}), \label{eq:lgmil_framework}
\end{align}
which is a convex relaxation of Equation~(\ref{eq:dual-mil}).

\subsubsection{Algorithm}

Similar to semi-supervised learning, the cutting plane algorithm is used for
solving Equation~(\ref{eq:lgmil_framework}). Recall that there are two issues in the
use of cutting-plane algorithms, namely, efficient multiple label-kernel
learning and the finding of a violated label assignment. For the first issue,
suppose that the current $\C$ is $\{\ss_1,\dots,\ss_T\}$, the MKL problem in Equation~(\ref{eq:lgmil_framework}) corresponds to the following primal problem:
\begin{eqnarray}\label{eq:mil_mkl}
\min\limits_{\bmu\in\M,\bW=[\w_1;\dots;\w_T],\bxi} &&
\frac{1}{2}\sum_{t=1}^{T}\frac{1}{\mu_{t}} ||\w_t||^{2} +
C_1 \sum_{i=1}^p \xi_i + C_2 \sum_{i=p+1}^{m}\sum_{j=1}^{m_i} \xi_{s(i,j)} \\
\mbox{s.t.} && \sum_{t=1}^{T}\left(\sum_{j=1}^{m_i}\w_t'
d_{i,j}^{t}\phi(\x_{i,j})\right) \geq 1 - \xi_i, \; \xi_i \geq 0, \; i  = 1,\dots, p, \nonumber\\
&& \!\!\!\!\!\!\!\!\!\!\!\!\!\!\!\! - \sum_{t=1}^{T}\w_t'\phi(\x_{s(i,j)}) \geq 1 - \xi_{s(i,j)}, \; \xi_{s(i,j)} \geq 0, \; i= p+1,\dots,m; \; j = 1, \dots, m_i. \nonumber
\end{eqnarray}
Therefore, we can still apply the MKLGL algorithm to solve MKL problem in
Equation~(\ref{eq:lgmil_framework}) efficiently. As for the second issue, one needs to solve the following problem:
\begin{align*}
\min_{\ss \in \D} & \;\;\;\;  \1'\ba - \frac{1}{2} (\ba \odot
\hat{\y})'\K^{\ss} (\ba \odot \hat{\y}),
\end{align*}
which is equivalent to
\begin{equation*}
\max\limits_{\ss\in \D} \;\;\; \sum\nolimits_{i,j=1}^{q}\alpha_i \alpha_j
\hat{y}_i \hat{y}_j (\bpsi^{\ss}_i)'(\bpsi^{\ss}_j).
\end{equation*}
According to the definition of $\bpsi$ in Equation~(\ref{eq:phi}), this can be rewritten as
\begin{align*}
\max\limits_{\ss\in \D} \; \left\|\sum_{i=1}^{p}\alpha_i
\sum_{j=1}^{m_i}d_{i,j}\phi(\x_{i,j}) \! - \!\!\!\!
\sum_{i=p+1}^{m}\sum_{j=1}^{m_i}\alpha_{s(i,j)} \phi(\x_{s(i,j)})\right\|^2,
\end{align*}
which can be reformulated as
\begin{align}\label{eq:concave_QP}
\max\limits_{{\ss}\in \D} \;\;\;\; \ss'\bH\ss + \bt'\ss,
\end{align}
where $\bH \in \R^{J_{p} \times J_{p}}$ and $\bt \in \R^{J_p}$. Let $v(i,j) = J_{i-1}+j$, $i \in 1,\dots, p, j \in 1, \dots, m_i$, we have $H_{v(i,j), v(\hat{i},\hat{j})} = \alpha_{i}\alpha_{\hat{i}}\phi(\x_{i,j})'\phi(\x_{\hat{i},\hat{j}})$ and $\tau_{v(i,j)} = -2\alpha_{i}\phi(\x_{i,j})'(\sum_{i=p+1}^{m}\sum_{j=1}^{m_i}\alpha_{s(i,j)} \phi(\x_{s(i,j)}))$. It is easy to verify that $\bH$ is psd.

Equation~(\ref{eq:concave_QP}) is also a concave QP whose globally optimal solution, or equivalently the most violated $\ss$, is intractable in general. In the following, we adapt a variant of the simple yet efficient method proposed in Section~\ref{sec:violate} to find a violated $\ss$. Let $\bar{\ss} \in \C$, where $\C =\{{\ss}_1,\dots,{\ss}_T\}$,
be the following suboptimal solution of Equation~(\ref{eq:concave_QP}):
$\bar{\ss} = \argmax_{\ss \in \C} \ss'\bH\ss + \bt'\ss$.
Let $\ss^{*}$ be an optimal solution of the
following optimization problem
\begin{align}\label{eq:a_violated_d}
{\ss}^{*} = \argmax\limits_{\ss\in \D} \ss'\bH\bar{\ss} + \frac{\bt'\ss}{2}.
\end{align}

\begin{prop}
${\ss}^{*}$ is a violated label assignment when $({\ss}^{*})'\bH \bar{\ss}+  \frac{\bt'\ss^{*}}{2}> \bar{\ss}'\bH \bar{\ss}+\frac{\bt'\bar{\ss}}{2}$.
\end{prop}
\begin{proof}
From $({\ss}^{*})'\bH \bar{\ss}+  \frac{\bt'\ss^{*}}{2}> \bar{\ss}'\bH \bar{\ss}+\frac{\bt'\bar{\ss}}{2}$, we have ${\ss}^{*} \neq \bar{\ss}$.
Suppose that $({\ss}^{*})'\bH {\ss}^{*}+\bt'\ss^{*} \leq
\bar{\ss}'\bH\bar{\ss}+\bt'\bar{\ss}$. Then
\begin{eqnarray*}
\lefteqn{\Big(({\ss}^{*})'\bH {\ss}^{*} + \bt'\ss^{*}\Big)+
\Big(\bar{\ss}'\bH\bar{\ss}+\bt'\bar{\ss}\Big)-\Big[2({\ss}^{*})'\bH
\bar{\ss}+\bt'\bar{\ss}+\bt'\ss^{*} \Big]} \\
& \leq &
2\Big[\bar{\ss}'\bH\bar{\ss}+\bt'\bar{\ss} - ({\ss}^{*})'\bH
\bar{\ss}-\frac{\bt'\bar{\ss}}{2}-\frac{\bt'\ss^{*}}{2} \Big] < 0,
\end{eqnarray*}
which contradicts
\begin{eqnarray*}
\Big(({\ss}^{*})'\bH {\ss}^{*} + \bt'\ss^{*}\Big)+
\Big(\bar{\ss}'\bH\bar{\ss}+\bt'\bar{\ss}\Big)-\Big[2({\ss}^{*})'\bH
\bar{\ss}+\bt'\bar{\ss}+\bt'\ss^{*} \Big]
& = & ({\ss}^{*} - \bar{\ss})'\bH
({\ss}^{*}- \bar{\ss}) \\
& \geq & 0.
\end{eqnarray*}
So $({\ss}^{*})'\bH {\ss}^{*}+\bt'(\ss^{*}) > \bar{\ss}'\bH\bar{\ss}+\bt'\bar{\ss}$,
which indicates that ${\ss}^{*}$ is a violated label assignment.
\end{proof}

Similar to Equation~(\ref{eq:a_violated_y}), Equation~(\ref{eq:a_violated_d}) is also a linear integer program but with different constraints. We now show that the optimal ${\ss}^*$ in Equation~(\ref{eq:a_violated_d}) can still be solved via sorting. Notice that Equation~(\ref{eq:a_violated_d}) can be reformulated as
\begin{eqnarray}\label{eq:toy3}
\max\limits_{\ss } && \br'{\ss} \\
\mbox{s.t}. && \1'\dd_i = 1, \dd_i \in \{0,1\}^{m_i}, i = 1, \dots, p, \nonumber
\end{eqnarray}
where $\br = \bH \bar{\ss}+\frac{\bt}{2}$. As can be seen, $\dd_i$'s are decoupled in both the objective and constraints of Equation~(\ref{eq:toy3}). Therefore, one can obtain its optimal solution by solving the $p$ subproblems individually
\begin{eqnarray*}
\max\limits_{\dd_i} && \sum_{j=1}^{m_i} r_{J_{i-1}+j} d_{i,j} \\
\mbox{s.t}. && \1'\dd_i = 1, \dd_i \in \{0,1\}^{m_i}. \nonumber
\end{eqnarray*}
It is evident that the optimal $\dd_i$ can be obtained by assigning
$d_{i,\hat{i}}=1$,
where $\hat{i}$ is the
index of the largest element among $[r_{J_{i-1}+1},\dots,r_{J_{i-1}+m_i}]$,
and the rest to zero. Similar to semi-supervised learning, the complexity to
find a violated $\ss$ scales as $O(N^2)$ (resp. $O(N\log N)$) when the nonlinear (resp.
linear) kernel is used, and so is computationally efficient.

On prediction, each instance $\x$ can be treated as a bag,  and its output from the \lgsvm\
is given by $f(\x) = \sum_{t=1}^{T}\w_t'\phi(\x)$. Algorithm \ref{alg:MIL}
summarizes the pseudo codes of \lgsvm\ for multi-instance learning.

\begin{algorithm}[htbp]
\caption{\lgsvm\ for multi-instance learning.}
\begin{algorithmic}[1]
\STATE Initialize ${\ss^{*}}$ and $\C = \emptyset$.
\REPEAT
\STATE Update $\C \leftarrow \{{\ss}^{*}\} \bigcup \C$.
\STATE Obtain the optimal $\{\bmu,\bW\}$ or $\ba$ from Equation~(\ref{eq:mil_mkl}).
\STATE Find the optimal solution ${\ss}^{*}$ of Equation~(\ref{eq:a_violated_d}).
\UNTIL{$G(\ba,\ss^{*}) > \min_{\ss \in \C} G(\ba,\ss) - \eps$ or the decrease of objective value is smaller than a threshold. }
\STATE Output $f(\x) = \sum_{t=1}^{T}\w_t'\phi(\x)$ as the prediction function.
\end{algorithmic}
\label{alg:MIL}
\end{algorithm}

\subsection{Clustering}
\label{subsec:clustering}

In this section, we consider the third weakly labeled learning task, namely,
clustering, where all the class labels are unknown. Similar to semi-supervised
learning, one can obtain a trivially ``optimal'' solution with infinite
margin by assigning all patterns to the same cluster. To prevent such a useless
solution, \citet{xu2005mmc} introduced a class balance constraint
\[ -\b \leq \1'\hat{\y}\leq \b, \]
where $\hat{\y}=[\hat{y}_1,\dots, \hat{y}_N]'$ is the vector of unknown labels, and $\beta\geq 0$ is a user-defined constant controlling the class imbalance.

Let $\Omega(f) = \frac{1}{2}\|\w\|_2^2$ and $\ell_{f}(\md)$ be the sum of hinge
losses on the individual examples. Equation~(\ref{eq:weakly labeled}) then leads to
\begin{eqnarray}
\min_{\hat{\y} \in \B}\min\limits_{\w, \bxi} && \frac{1}{2}||\w||_2^2 + C \sum_{i=1}^{N} \xi_i
\label{eq:mmc_primal}\\
\mbox{s.t} && \hat{y}_i \w'\phi(\x_i) \geq 1 - \xi_i, \;\; \xi_i \geq 0, \;\; i = 1
\dots, N, \nonumber
\end{eqnarray}
where $\B = \{\hat{\y} \;|\; \hat{y}_i \in \{+1,-1\}, i=1,\dots, N; -\b \leq \1'\hat{\y}\leq \b\}$.
The inner minimization problem is usually rewritten in its dual
\begin{eqnarray}
\min_{\hat{\y} \in \B} \max_{\ba} && \sum_{i=1}^{N}\alpha_i - \frac{1}{2}
\sum_{i,j=1}^{N}\a_i\a_j\left(\hat{y}_i\hat{y}_j \phi(\x_i)'\phi(\x_j)\right)
\label{eq:mmc_dual} \\
\mbox{s.t.}&& C \geq \a_i\geq 0, \;\; i = 1 \dots, N, \nonumber
\end{eqnarray}
where $\alpha_i$ is the dual variable for each inequality constraint in Equation~(\ref{eq:mmc_primal}). Let $\ba=[\a_1,\cdots,\a_N]'$ be the vector of dual variables, and $\A=\{\ba \; \big| \; C \1 \geq \ba\geq \0\}$. Then Equation~(\ref{eq:mmc_dual}) can be rewritten in matrix form as
\begin{eqnarray} \label{eq:mmc_dual3}
&\min_{\hat{\y} \in \B} \max_{\ba\in\A} &  G(\ba,\hat{\y}) := \1'\ba - \frac{1}{2}\ba'\Big(\K\odot \hat{\y}\hat{\y}' \Big) \ba.
\end{eqnarray}
This, however, is still a mixed integer programming problem.

With Proposition 1, we have
\begin{align}
\min_{\bmu\in \M} \max_{\ba\in\A}  \;\; \1'\ba-
\frac{1}{2}\ba'\Big(\sum_{t:\hat{\y}_t\in\B}\mu_t\K\odot
\hat{\y}_t\hat{\y}_t' \Big) \ba \label{eq:mmc_dual_maxmin}
\end{align}
as a convex relaxation of Equation~(\ref{eq:mmc_dual3}). Note that $G(\ba,\hat{\y})$ can be reformulated by
$\bar{G}(\ba,\bM_{\y})=\1'\ba - \frac{1}{2}\ba'\Big(\K\odot \bM_{\y}
\Big) \ba$, where $\bar{G}$ is concave in $\ba$ and linear in $\bM_{\y}$. Hence, according to Theorem 1, \lgsvm\
is at least as tight as the SDP relaxation in \citet{xu2005mmc}.

\subsubsection{Algorithm}

The cutting plane algorithm can still be applied for clustering. Similar to semi-supervised learning, the MKL can be formulated as the following primal problem:
\begin{eqnarray}\label{eq:mmc_mkl}
\min\limits_{\bmu\in\M,\bW=[\w_1;\dots;\w_T],\bxi} &&
\frac{1}{2}\sum_{t=1}^{T}\frac{1}{\mu_{t}} ||\w_t||^{2} +
C \sum_{i=1}^N\xi_i \\
\mbox{s.t.} && \sum_{t=1}^{T}\hat{y}_{ti}\w_t'\phi(\x_i) \geq 1 -
\xi_i, \;\; \xi_i \geq 0, \;\;i = 1,\dots, N, \nonumber
\end{eqnarray}
and its dual is
\begin{eqnarray*}
\min\limits_{\bmu \in \M} \max\limits_{\ba \in \A} & \1'\ba-\frac{1}{2} \ba'\Big(\sum_{t=1}^{T}\mu_t
\K \odot \hat{\y}_t \hat{\y}_t' \Big)\ba,
\end{eqnarray*}
which is the same as Equation~(\ref{eq:mmc_dual_maxmin}). Therefore, MKLGL algorithm can still be applied for solving the MKL problem in Equation~(\ref{eq:mmc_dual_maxmin}) efficiently.

As for finding a violated label assignment, let $\bar{\y} \in \C$ be
\[
\bar{\y} = \mathop{\arg\max}\nolimits_{\hat{\y}\in \C} \hat{\y}'\bH\hat{\y},
\]
where $\bH = \K \odot (\ba \ba')$ is a positive semidefinite matrix.
Consider an optimal solution of the following optimization problem
\begin{align}\label{eq:a_violated_y_2}
{\y}^{*} = \argmax\limits_{\hat{\y}\in \B} \hat{\y}'\bH\bar{\y}.
\end{align}
With
Proposition~\ref{prop:violate},
we obtain that $\y^{*}$ is a violated label assignment
if $\bar{\y}'\bH{\y}^{*} \geq \bar{\y}'\bH\bar{\y}$.

Note that Equation~(\ref{eq:a_violated_y_2}) is a linear program for $\hat{\y}$ and can be formulated as
\begin{eqnarray}\label{eq:toy4}
\max\limits_{\hat{\y}} && \br'\hat{\y} \\
\mbox{s.t}. && -\beta \leq \hat{\y}'\1 \leq \beta, \hat{\y} \in
\{-1,+1\}^{N}, \nonumber
\end{eqnarray}
where $\br = \bH \bar{\y}$. From Proposition~\ref{prop:ordering},
we can solve
Equation~(\ref{eq:toy4}) by first sorting $r_i$'s in ascending order. The label assignment of
$\hat{y}_i$'s aligns with the sorted values of $r_i$'s. To satisfy the balance
constraint $-\beta\leq \1'\hat{\y}\leq \beta$, the first $\frac{N-\beta}{2}$ of
$\hat{y}_i$'s are assigned $-1$, the last $\frac{N-\beta}{2}$ of them are
assigned 1, and the rest $\hat{y}_i$'s are assigned $-1$ (resp. $1$) if the corresponded $r_i$'s are negative (resp. non-negative). It is easy to verify that such an assignment satisfies the balance constraint and the objective $\br'\hat{\y}$ is maximized.
Similar to
semi-supervised learning, the complexity to find a violated label scales as $O(N^2)$
(resp. $O(N\log N)$) when the nonlinear (resp. linear) kernel is used, and so is computationally efficient.
Finally, we use $f(\x) = \sum_{t=1}^{T}\w_t'\x$ as the
prediction function.
Algorithm \ref{alg:MMC} summarizes the pseudo codes of \lgsvm\ for clustering.

\begin{algorithm}[htbp]
\caption{\lgsvm\ for clustering.}
\begin{algorithmic}[1]
\STATE Initialize $\hat{\y}$ and $\C = \emptyset$.
\REPEAT
\STATE Update $\C \leftarrow \{{\y}^{*}\} \bigcup \C$.
\STATE Obtain the optimal $\{\bmu,\bW\}$ or $\ba$ from Equation~(\ref{eq:mmc_mkl}).
\STATE Find the optimal solution ${\y}^{*}$ of Equation~(\ref{eq:a_violated_y_2}).
\UNTIL{$G(\ba,\y^{*}) > \min_{\y \in \C} G(\ba,\y) - \eps$ or the decrease of objective value is smaller than a threshold. }
\STATE Output $f(\x) = \sum_{t=1}^{T}\w_t'\phi(\x)$ as the prediction function.
\end{algorithmic}
\label{alg:MMC}
\end{algorithm}

\section{Experiments}
\label{sec:expt}

In this section, comprehensive evaluations are performed to verify the effectiveness of the proposed \lgsvm. Experiments are conducted on all the three aforementioned weakly labeled learning tasks: semi-supervised learning (Section \ref{subsec:ssl-expt}), multi-instance learning (Section \ref{subsec:mil-expt}) and clustering (Section
\ref{subsec:clustering-expt}). For nonlinear kernel, the \lgsvm\ adapts
$\nu$-SVM with square hinge loss \citep{tsang2006cvm} and is implemented using the LIBSVM \citep{fan2005wss};
For linear kernel, it adapts standard SVM without offset and is implemented using the LIBLINEAR \citep{hsieh2008dcd}. Experiments are run on a 3.20GHz Intel Xeon(R)2 Duo PC running Windows 7 with 8GB main memory. For all the other methods that will be used for comparison, the default stopping criteria in the corresponding packages are used. For the \lgsvm, both the $\eps$ and stopping threshold
in Algorithm~\ref{alg:LGSVM} are set to $10^{-3}$.

\subsection{Semi-Supervised Learning}
\label{subsec:ssl-expt}

We first evaluate the \lgsvm\ on semi-supervised learning with a large
collection of real-world data sets. 16 UCI data sets, which cover
a wide range of properties, and 2 large-scale data
sets\footnote{Data sets can be found at \url{http://www.csie.ntu.edu.tw/~cjlin/libsvmtools/datasets/binary.html}.} are used.
Table~\ref{table:UCI} shows some statistics of these data sets.

\begin{table}[htbp]
\smallskip \small
\begin{center}
\begin{tabular}{lccc||lccccccc}
\hline\noalign{\smallskip}
& Data & \# Instances & \# Features & & Data & \# Instances & \# Features \\
\hline
1 & \textit{Echocardiogram}  & 132 & 8 & 10 &\textit{Clean1}  & 476 & 166   &  \\
2 & \textit{House} & 232 & 16  &  11  & \textit{Isolet} & 600 & 51 \\
3 & \textit{Heart} & 270 & 9  & 12  &\textit{Australian}  & 690 & 42  \\
4 & \textit{Heart-stalog}  & 270 & 13  & 13  &\textit{Diabetes}  & 768 & 8 \\
5 &\textit{Haberman}  & 306 & 14  & 14   &\textit{German}  & 1,000 & 59 \\
6 &\textit{LiveDiscorders}  & 345 & 6  & 15  &\textit{Krvskp}  & 3,196 & 36 \\
7 &\textit{Spectf}  & 349 & 44  & 16  & \textit{Sick}  & 3,772 & 31 \\
8 & \textit{Ionosphere}  & 351 & 34  & 17  & \textit{real-sim} & 72,309 & 20,958 \\
9 &\textit{House-votes}  & 435 & 16   & 18  &\textit{rcv1} & 677,399 & 47,236 \\
\noalign{\smallskip}\hline
\end{tabular}
\end{center}
\caption{Data sets used in the experiments.}\label{table:UCI}
\end{table}

\subsubsection{Small-Scale Experiments}
\label{sec:small}

For each UCI data set, $75\%$ of the examples
are randomly chosen for training, and the rest for testing.
We investigate
the performance of each approach with varying amount of labeled data
(namely, $5\%$, $10\%$ and $15\%$ of all the labeled data).
The whole
setup is repeated $30$ times and the average accuracies (with standard
deviations) on the test set are reported.

We compare \lgsvm\ with 1) the standard SVM (using labeled data only), and three
state-of-the-art semi-supervised SVMs (S$^3$VMs), namely 2) Transductive SVM (TSVM)\footnote{Transductive SVM can be found at \url{http://svmlight.joachims.org/}.}
\citep{Joachims1999}; 3)
Laplacian SVM (LapSVM)\footnote{Laplacian SVM can be found at \url{http://manifold.cs.uchicago.edu/manifold_regularization/software.html}.}
\citep{Belkin:Niyogi:Sindhwani2006};
and 4)
UniverSVM
(USVM)\footnote{UniverSVM can be found at \url{http://mloss.org/software/view/19/}.}
\citep{collobert2006lst}.
Note that TSVM and USVM adopt the same
objective as \lgsvm, but with different optimization strategies (local search
and constrained convex-concave procedure, respectively). LapSVM is another
S$^3$VM based on the manifold assumption
\citep{Belkin:Niyogi:Sindhwani2006}. The SDP-based
S$^3$VMs \citep{xu2005semi,de2006semi} are not compared, as they
do not converge after $3$ hours on even the smallest data set
(\textit{Echocardiogram}).

Parameters of the different methods are set as follows.
$C_1$ is fixed at 1 and $C_2$ is selected in the range
$\{0.001,0.005,0.01,0.05,0.1, 0.5,1\}$. The linear and Gaussian kernels
are used for all SVMs, where the width $\sigma$ of the Gaussian kernel
$k(\x,\hat{\x})=\exp(-||\x-\hat{\x}||^2/2\sigma^2)$ is picked from
$\{0.25\sqrt{\gamma},0.5\sqrt{\gamma}$
$,\sqrt{\gamma},2\sqrt{\gamma},4\sqrt{\gamma}\}$, with $\gamma$ being the
average distance between all instance pairs. The initial label assignment of
\lgsvm\ is obtained from the predictions of a standard SVM. For LapSVM,
the number of nearest neighbors
in the underlying data graph
is selected from $\{3,5,7,9\}$. All
parameters are determined by using the five-fold cross-validated accuracy.

Table~\ref{table:uci-ssl-5} shows the results
on the UCI data sets with $5\%$
labeled examples.
As can be seen, \lgsvm\ obtains highly
competitive performance with the other methods, and achieves the best
improvement against
SVM in terms of both the counts of ($\#\text{wins}- \#\text{loses}$) as well as
average accuracy. The Friedman test \citep{demvsar2006statistical}
shows that both \lgsvm\ and USVM perform significantly better
than SVM
at the 90\% confidence level, while TSVM and LapSVM do not.

As can be seen, there are cases where unlabeled data cannot help for TSVM, USVM and
\lgsvm. Besides the local minimum problem, another possible reason may be
that there are multiple large margin separators coinciding well with labeled data and the labeled examples are too few to provide a reliable selection for these separators \citep{li2011s4vm}. Moreover, overall, LapSVM cannot obtain good performance, which may be due
to that the manifold assumption does not hold on these data \citep{chapelle2006semi}.

Tables~\ref{table:uci-ssl-10} and \ref{table:uci-ssl-15} show the results on the UCI
data sets with $10\%$ and $15\%$ labeled examples, respectively. As
can be seen, as the number of labeled examples increases, SVM gets much better performance. As a result,
both TSVM and USVM cannot beat the SVM. On the other
hand, the Friedman test shows that \lgsvm\ still performs significantly better than SVM with $10\%$
labeled examples at the 90\% confidence level.
With $15\%$ labeled examples,
no S$^3$VM performs significantly better than SVM.

\begin{table}[t]
\begin{center}\scriptsize
\begin{tabular}{l|r|rrrrrccccc}
\hline\noalign{\smallskip}
Data & SVM  &  TSVM &   LapSVM &  USVM & \lgsvm  \\
\hline\noalign{\smallskip}
\textit{Echocardiogram} & 0.80 $\pm$ 0.07 (2.5) & 0.74 $\pm$ 0.08 (4) & 0.64 $\pm$ 0.22 (5) & \textbf{0.81 $\pm$ 0.06} (1) & 0.80 $\pm$ 0.07 (2.5) \\
\textit{House} & \textbf{0.90 $\pm$ 0.04} (3) & \textbf{0.90 $\pm$ 0.05} (3) & \textbf{0.90 $\pm$ 0.04} (3) & \textbf{0.90 $\pm$ 0.03} (3) & \textbf{0.90 $\pm$ 0.04} (3) \\
\textit{Heart} & 0.70 $\pm$ 0.08 (5) & 0.75 $\pm$ 0.08 (3) & 0.73 $\pm$ 0.09 (4) & 0.76 $\pm$ 0.07 (2) & \textbf{0.77 $\pm$ 0.08} (1) \\
\textit{Heart-statlog} & 0.73 $\pm$ 0.10 (4.5) & \textbf{0.75 $\pm$ 0.10} (1.5) & 0.74 $\pm$ 0.11 (3) & \textbf{0.75 $\pm$ 0.12} (1.5) & 0.73 $\pm$ 0.12 (4.5) \\
\textit{Haberman} & 0.65 $\pm$ 0.07 (3) & 0.61 $\pm$ 0.06 (4) & 0.57 $\pm$ 0.11 (5) & \textbf{0.75 $\pm$ 0.05} (1.5) & \textbf{0.75 $\pm$ 0.05} (1.5) \\
\textit{LiverDisorders} & 0.56 $\pm$ 0.05 (2) & 0.55 $\pm$ 0.05 (3.5) & 0.55 $\pm$ 0.05 (3.5) & \textbf{0.59 $\pm$ 0.05} (1) & 0.53 $\pm$ 0.07 (5) \\
\textit{Spectf} & 0.73 $\pm$ 0.05 (2) & 0.68 $\pm$ 0.10 (4) & 0.61 $\pm$ 0.08 (5) & \textbf{0.74 $\pm$ 0.05} (1) & 0.70 $\pm$ 0.07 (3) \\
\textit{Ionosphere} & 0.67 $\pm$ 0.06 (4) & \textbf{0.82 $\pm$ 0.11} (1) & 0.65 $\pm$ 0.05 (5) & 0.77 $\pm$ 0.07 (2) & 0.70 $\pm$ 0.08 (3) \\
\textit{House-votes} & 0.88 $\pm$ 0.03 (3) & \textbf{0.89 $\pm$ 0.05} (1.5) & 0.87 $\pm$ 0.03 (4) & 0.83 $\pm$ 0.03 (5) & \textbf{0.89 $\pm$ 0.03} (1.5) \\
\textit{Clean1} & 0.58 $\pm$ 0.06 (4) & 0.60 $\pm$ 0.08 (3) & 0.54 $\pm$ 0.05 (5) & \textbf{0.65 $\pm$ 0.05} (1) & 0.63 $\pm$ 0.07 (2) \\
\textit{Isolet} & 0.97 $\pm$ 0.02 (3) & \textbf{0.99 $\pm$ 0.01} (1) & 0.97 $\pm$ 0.02 (3) & 0.70 $\pm$ 0.09 (5) & 0.97 $\pm$ 0.02 (3) \\
\textit{Australian} & 0.79 $\pm$ 0.05 (4) & \textbf{0.82 $\pm$ 0.07} (1) & 0.78 $\pm$ 0.08 (5) & 0.80 $\pm$ 0.05 (3) & 0.81 $\pm$ 0.04 (2) \\
\textit{Diabetes} & 0.67 $\pm$ 0.04 (4) & 0.67 $\pm$ 0.04 (4) & 0.67 $\pm$ 0.04 (4) & \textbf{0.70 $\pm$ 0.03} (1) & 0.69 $\pm$ 0.03 (2)\\
\textit{German} & \textbf{0.70 $\pm$ 0.03} (2) & 0.69 $\pm$ 0.03 (4) & 0.62 $\pm$ 0.05 (5) & \textbf{0.70 $\pm$ 0.02} (2) & \textbf{0.70 $\pm$ 0.02} (2) \\
\textit{Krvskp} & 0.91 $\pm$ 0.02 (3.5) & \textbf{0.92 $\pm$ 0.03} (1.5) & 0.80 $\pm$ 0.02 (5) & 0.91 $\pm$ 0.03 (3.5) & \textbf{0.92 $\pm$ 0.02} (1.5) \\
\textit{Sick} & \textbf{0.94 $\pm$ 0.01} (2) & 0.89 $\pm$ 0.01 (5) & 0.90 $\pm$ 0.02 (4) & \textbf{0.94 $\pm$ 0.01} (2) & \textbf{0.94 $\pm$ 0.01} (2) \\
\noalign{\smallskip}\hline
\multicolumn{2}{c|}{SVM: win/tie/loss}  & 5/7/4 & 8/7/1  & 2/9/5  & \textbf{3/6/7}  \\
\noalign{\smallskip}\hline
ave. acc. & 0.763 & 0.767 & 0.723 & 0.770 & \textbf{0.778} \\
ave. rank &  3.2188  &  2.8125  &  4.2813  &  2.2188  &  2.4688 \\
\noalign{\smallskip}\hline
\end{tabular}
\end{center}\vspace{-4mm}
\caption{Accuracies on the various data sets with 5\% labeled examples. The
best performance on each data set is bolded. The win/tie/loss counts (paired
$t$-test at $95\%$ significance level) are listed. The method with the largest
number of (\#wins - \#losses) against SVM as well as the best average accuracy
is also highlighted. Number in parentheses denotes the ranking
(computed as in \citealt{demvsar2006statistical})
of each method on the data set.}\label{table:uci-ssl-5} \vspace{-0mm}
\end{table}

\begin{table}[h]
\smallskip
\begin{center}\scriptsize
\begin{tabular}{l|r|rrrrrccccc}
\hline\noalign{\smallskip}
Data & SVM  &  TSVM &   LapSVM &  USVM & \lgsvm  \\
\hline\noalign{\smallskip}
\textit{Echocardiogram} & 0.81 $\pm$ 0.05 (2.5) & 0.76 $\pm$ 0.12 (4) & 0.69 $\pm$ 0.14 (5) & \textbf{0.82 $\pm$ 0.05} (1) & 0.81 $\pm$ 0.05 (2.5) \\
\textit{House} & 0.90 $\pm$ 0.04 (2.5) & \textbf{0.92 $\pm$ 0.05} (1) & 0.89 $\pm$ 0.04 (4) & 0.83 $\pm$ 0.03 (5) & 0.90 $\pm$ 0.04 (2.5) \\
\textit{Heart} & 0.76 $\pm$ 0.05 (3.5) & 0.75 $\pm$ 0.05 (5) & 0.76 $\pm$ 0.06 (3.5) & \textbf{0.78 $\pm$ 0.05} (1.5) & \textbf{0.78 $\pm$ 0.04} (1.5)\\
\textit{Heart-statlog} & 0.79 $\pm$ 0.03 (4) & 0.74 $\pm$ 0.05 (5) & 0.80 $\pm$ 0.04 (2.5) & 0.80 $\pm$ 0.04 (2.5) & \textbf{0.81 $\pm$ 0.04} (1) \\
\textit{Haberman} & \textbf{0.75 $\pm$ 0.04} (2) & 0.60 $\pm$ 0.07 (4.5) & 0.60 $\pm$ 0.07 (4.5) & \textbf{0.75 $\pm$ 0.04} (2) & \textbf{0.75 $\pm$ 0.04} (2) \\
\textit{LiverDisorders} & \textbf{0.59 $\pm$ 0.06} (1) & 0.57 $\pm$ 0.05 (2.5) & 0.55 $\pm$ 0.06 (4) & 0.53 $\pm$ 0.06 (5) & 0.57 $\pm$ 0.05 (2.5)\\
\textit{Spectf} & 0.74 $\pm$ 0.05 (2) & \textbf{0.76 $\pm$ 0.06} (1) & 0.64 $\pm$ 0.06 (5) & 0.72 $\pm$ 0.06 (3.5) & 0.72 $\pm$ 0.07 (3.5) \\
\textit{Ionosphere} & 0.78 $\pm$ 0.07 (4) & \textbf{0.90 $\pm$ 0.04} (1) & 0.66 $\pm$ 0.06 (5) & 0.88 $\pm$ 0.05 (2) & 0.82 $\pm$ 0.05 (3) \\
\textit{House-votes} & \textbf{0.92 $\pm$ 0.03} (1.5) & 0.91 $\pm$ 0.03 (3.5) & 0.88 $\pm$ 0.04 (5) & 0.91 $\pm$ 0.03 (3.5) & \textbf{0.92 $\pm$ 0.03} (1.5)\\
\textit{Clean1} & 0.69 $\pm$ 0.05 (3.5) & 0.71 $\pm$ 0.05 (2) & 0.63 $\pm$ 0.07 (5) & \textbf{0.72 $\pm$ 0.05} (1) & 0.69 $\pm$ 0.04 (3.5) \\
\textit{Isolet} & 0.99 $\pm$ 0.01 (2.5) & \textbf{1.00 $\pm$ 0.01} (1) & 0.96 $\pm$ 0.02 (4) & 0.52 $\pm$ 0.03 (5) & 0.99 $\pm$ 0.01 (2.5) \\
\textit{Australian} & 0.81 $\pm$ 0.03 (5) & \textbf{0.84 $\pm$ 0.03} (1.5) & 0.82 $\pm$ 0.04 (4) & \textbf{0.84 $\pm$ 0.03} (1.5) & 0.83 $\pm$ 0.03 (3) \\
\textit{Diabetes} & 0.70 $\pm$ 0.03 (4.5) & 0.70 $\pm$ 0.05 (4.5) & 0.71 $\pm$ 0.04 (3) & 0.72 $\pm$ 0.03 (2) & \textbf{0.74 $\pm$ 0.03} (1) \\
\textit{German} & 0.67 $\pm$ 0.03 (3.5) & 0.67 $\pm$ 0.03 (3.5) & 0.66 $\pm$ 0.04 (5) & \textbf{0.70 $\pm$ 0.02} (1.5) & \textbf{0.70 $\pm$ 0.02} (1.5) \\
\textit{Krvskp} & 0.93 $\pm$ 0.01 (3) & 0.93 $\pm$ 0.01 (3) & 0.86 $\pm$ 0.04 (5) & 0.93 $\pm$ 0.01 (3) & \textbf{0.94 $\pm$ 0.01} (1) \\
\textit{Sick} & \textbf{0.93 $\pm$ 0.01} (2) & 0.89 $\pm$ 0.01 (5) & 0.92 $\pm$ 0.01 (4) & \textbf{0.93 $\pm$ 0.01} (2) & \textbf{0.93 $\pm$ 0.01} (2)\\
\noalign{\smallskip}\hline
\multicolumn{2}{c|}{SVM: win/tie/loss}  & 5/8/3 & 10/5/1  & 5/6/5  & \textbf{0/9/7}  \\
\noalign{\smallskip}\hline
avg. acc. & 0.799 & 0.789 & 0.753 & 0.774 & \textbf{0.807} \\
avg. rank &  2.9375  &  3.0000 &   4.2813  &  2.6250  &  2.1563 \\
\noalign{\smallskip}\hline
\end{tabular}
\end{center}\vspace{-3mm}
\caption{Accuracies on the various data sets with 10\% labeled examples.}\label{table:uci-ssl-10}
\end{table}

\begin{table}[h]
 \smallskip
\begin{center}\scriptsize
\begin{tabular}{l|r|rrrrrccccc}
\hline\noalign{\smallskip}
Data & SVM  &  TSVM &   LapSVM &  USVM & \lgsvm  \\
\hline\noalign{\smallskip}
\textit{echocardiogram} & 0.83 $\pm$ 0.04 (2.5) & 0.76 $\pm$ 0.07 (4) & 0.75 $\pm$ 0.08 (5) & \textbf{0.85 $\pm$ 0} (1) & 0.83 $\pm$ 0.04 (2.5) \\
\textit{house} & 0.92 $\pm$ 0.04 (2.5) & \textbf{0.94 $\pm$ 0.04} (1) & 0.83 $\pm$ 0.11 (5) & 0.91 $\pm$ 0.04 (4) & 0.92 $\pm$ 0.03 (2.5) \\
\textit{heart} & 0.78 $\pm$ 0.06 (3) & 0.78 $\pm$ 0.05 (3) & \textbf{0.79 $\pm$ 0.05} (1) & 0.78 $\pm$ 0.07 (3) & 0.78 $\pm$ 0.06 (3) \\
\textit{heart-statlog} & 0.76 $\pm$ 0.06 (2) & 0.74 $\pm$ 0.06 (4) & \textbf{0.79 $\pm$ 0.05} (1) & 0.73 $\pm$ 0.07 (5) & 0.75 $\pm$ 0.06 (3) \\
\textit{haberman} & 0.72 $\pm$ 0.03 (3) & 0.62 $\pm$ 0.07 (5) & 0.63 $\pm$ 0.11 (4) & \textbf{0.74 $\pm$ 0} (1.5) & \textbf{0.74 $\pm$ 0} (1.5) \\
\textit{liverDisorders} & \textbf{0.61 $\pm$ 0.05} (1) & 0.54 $\pm$ 0.06 (4) & 0.53 $\pm$ 0.07 (5) & 0.58 $\pm$ 0 (2) & 0.56 $\pm$ 0.06 (3) \\
\textit{spectf} & 0.77 $\pm$ 0.03 (2) & \textbf{0.79 $\pm$ 0.04} (1) & 0.6 $\pm$ 0.1 (5) & 0.74 $\pm$ 0 (4) & 0.75 $\pm$ 0.06 (3) \\
\textit{ionosphere} & 0.76 $\pm$ 0.04 (5) & \textbf{0.9 $\pm$ 0.04} (1) & 0.83 $\pm$ 0.04 (4) & 0.89 $\pm$ 0.04 (2) & 0.84 $\pm$ 0.03 (3) \\
\textit{house-votes} & \textbf{0.92 $\pm$ 0.02} (1.5) & \textbf{0.92 $\pm$ 0.03} (1.5) & 0.9 $\pm$ 0.03 (3) & 0.83 $\pm$ 0.03 (5) & 0.89 $\pm$ 0.02 (4)\\
\textit{clean1} & 0.71 $\pm$ 0.04 (4) & 0.74 $\pm$ 0.04 (2) & 0.63 $\pm$ 0.07 (5) & \textbf{0.76 $\pm$ 0.06} (1) & 0.72 $\pm$ 0.04 (3)\\
\textit{isolet} & 0.98 $\pm$ 0.01 (3.5) & \textbf{0.99 $\pm$ 0.01} (1.5) & 0.98 $\pm$ 0.01 (3.5) & 0.54 $\pm$ 0.02 (5) & \textbf{0.99 $\pm$ 0.01} (1.5) \\
\textit{australian} & \textbf{0.86 $\pm$ 0.02} (1.5) & 0.85 $\pm$ 0.03 (3) & 0.83 $\pm$ 0.02 (4.5) & 0.83 $\pm$ 0.03 (4.5) & \textbf{0.86 $\pm$ 0.03} (1.5)\\
\textit{diabetes} & \textbf{0.75 $\pm$ 0.03} (1.5) & 0.73 $\pm$ 0.02 (3.5) & 0.73 $\pm$ 0.03 (3.5) & 0.72 $\pm$ 0.04 (5) & \textbf{0.75 $\pm$ 0.03} (1.5) \\
\textit{german} & 0.71 $\pm$ 0.01 (2) & 0.7 $\pm$ 0.03 (3.5) & 0.68 $\pm$ 0.04 (5) & 0.7 $\pm$ 0.04 (3.5) & \textbf{0.72 $\pm$ 0.01} (1)\\
\textit{krvskp} & \textbf{0.95 $\pm$ 0.01} (1.5) & 0.93 $\pm$ 0.01 (4) & 0.91 $\pm$ 0.01 (5) & 0.94 $\pm$ 0.01 (3) & \textbf{0.95 $\pm$ 0.01} (1.5)\\
\textit{sick} & \textbf{0.94 $\pm$ 0} (2) & 0.9 $\pm$ 0.01 (4.5) & 0.9 $\pm$ 0.12 (4.5) & \textbf{0.94 $\pm$ 0} (2) & \textbf{0.94 $\pm$ 0} (2) \\
\noalign{\smallskip}\hline
\multicolumn{2}{c|}{SVM: win/tie/loss}  & 8/3/5 & 11/2/3  & 6/6/4  & \textbf{2/9/5}  \\
\noalign{\smallskip}\hline
avg. acc. & 0.809 & 0.801 & 0.771 & 0.780 & \textbf{0.811} \\
avg. rank &  2.4063  &  2.9063  &  4.0000  &  3.2188 &   2.3438 \\
\noalign{\smallskip}\hline
\end{tabular}
\end{center}\vspace{-3mm}
\caption{Accuracies on various data sets with 15\% labeled examples.}\label{table:uci-ssl-15}
\end{table}

Figure~\ref{fig:cputimes-UCI-ssl} compares the average CPU time of \lgsvm\ with the other S$^3$VMs different numbers of labeled examples.
As can be seen, TSVM is the slowest while USVM is the most efficient. \lgsvm\ is comparable to LapSVM.
Figure~\ref{fig:obj} shows the objective
values of \lgsvm\ on five representative UCI data sets.
We can observe that the number of
iterations is always fewer than 25.
As mentioned above, the SDP-based S$^3$VMs \citep{xu2005semi,de2006semi}, in contrast, cannot
converge in $3$ hours even on the smallest
data set \textit{Echocardiogram}. Hence, \lgsvm\ scales much better than these SDP-based approaches.

\begin{figure}[h]
\centering
\includegraphics[height=4.5 in]{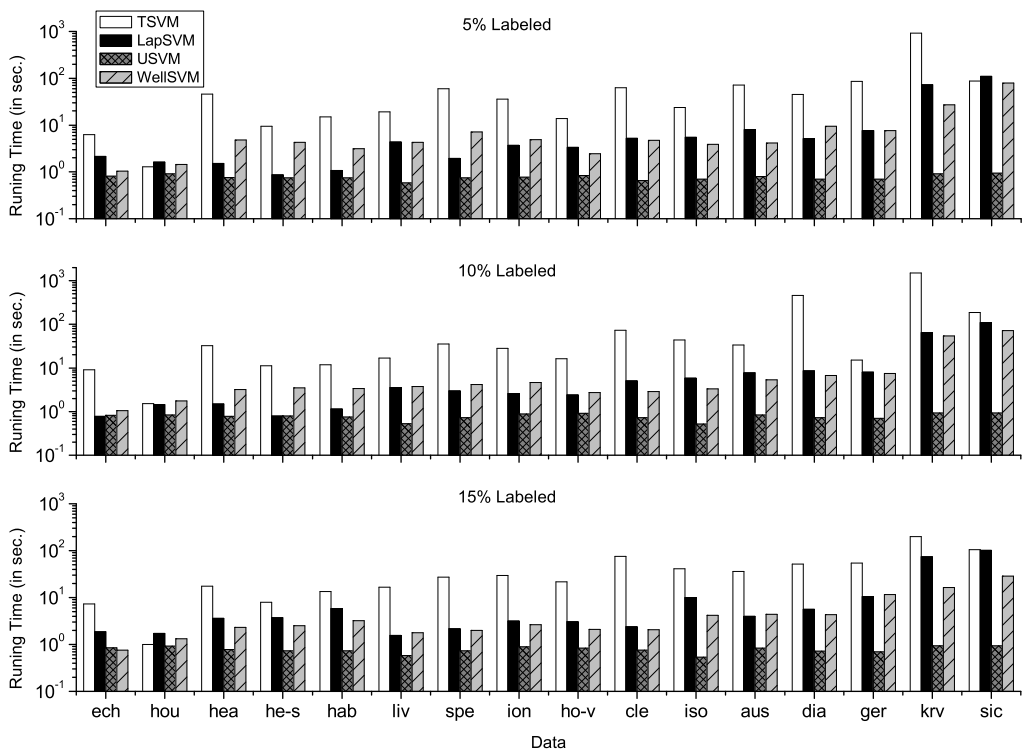}\vspace{-10mm}
\caption{CPU time on the UCI data sets.}\label{fig:cputimes-UCI-ssl}\vspace{-2mm}
\end{figure}

\begin{figure}
\centering
\includegraphics[width = 2.5 in]{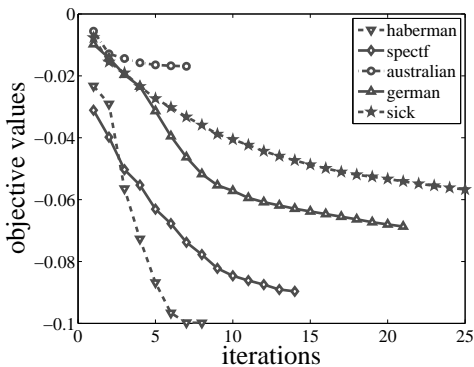}\vspace{-4mm}
\caption{Number of \lgsvm\ iterations on the UCI data sets.}\label{fig:obj}\vspace{-2mm}
\end{figure}

\subsubsection{Large-Scale Experiments}

In this section, we study the scalability of the proposed \lgsvm\ and other state-of-the-art approaches  on two large data sets, \textit{real-sim}
and \textit{RCV1}. The \textit{real-sim} data has 20,958 features and 72,309 instances. while the \textit{RCV1} data has 47,236 features and 677,399
instances. The linear kernel is used. The S$^3$VMs compared in Section~\ref{sec:small} are for general kernels and cannot converge in 24
hours. Hence, to conduct a fair comparison,
an efficient linear S$^3$VM solver, namely,
SVMlin\footnote{SVMlin can be found at \url{http://vikas.sindhwani.org/svmlin.html}.}
using deterministic annealing
\citep{sindhwani2006large}, is employed. All the parameters are determined in the same manner as in
Section~\ref{sec:small}.

In the first experiment, we study the performance at different numbers of unlabeled examples. Specifically, $1\%, 2\%, 5\%, 15\%, 35\%, 55\%$ and
$75\%$ of the data (with $50$ of them labeled)
are used for training, and $25\%$ of the data are for testing.
This is repeated $10$ times and the
average performance is reported.

Figure~\ref{fig:real-sim-comparison-ssl} shows the results. As can be seen, \lgsvm\ is always superior to SVMlin,
and achieves highly competitive or even better accuracy than the SVM as the number of unlabeled examples increases.
Moreover, \lgsvm\ is much faster than SVMlin. As the number of unlabeled examples increases, the difference becomes more prominent. This  is mainly because SVMlin employs gradient descent while \lgsvm\ (which is based on
LIBLINEAR \citep{hsieh2008dcd}) uses coordinate descent, which is known to be one of the fastest solvers for large-scale linear SVMs \citep{shalev2007pegasos}.

\begin{figure}[t]
\centering
\begin{minipage}[c]{2.5 in} \centering \includegraphics[width = 2.5 in]{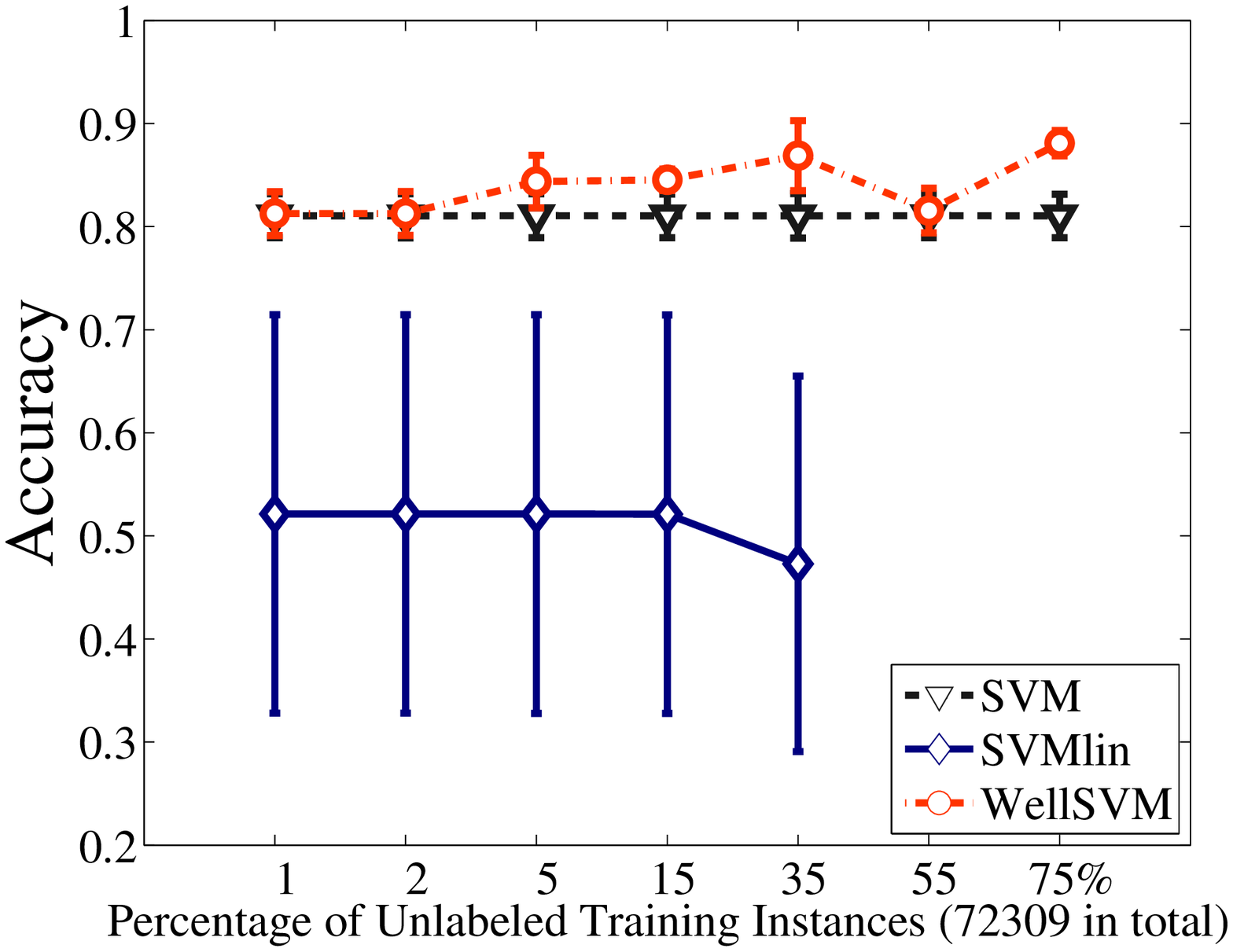}
\end{minipage}
\begin{minipage}[c]{2.5 in} \centering \includegraphics[width = 2.5 in]{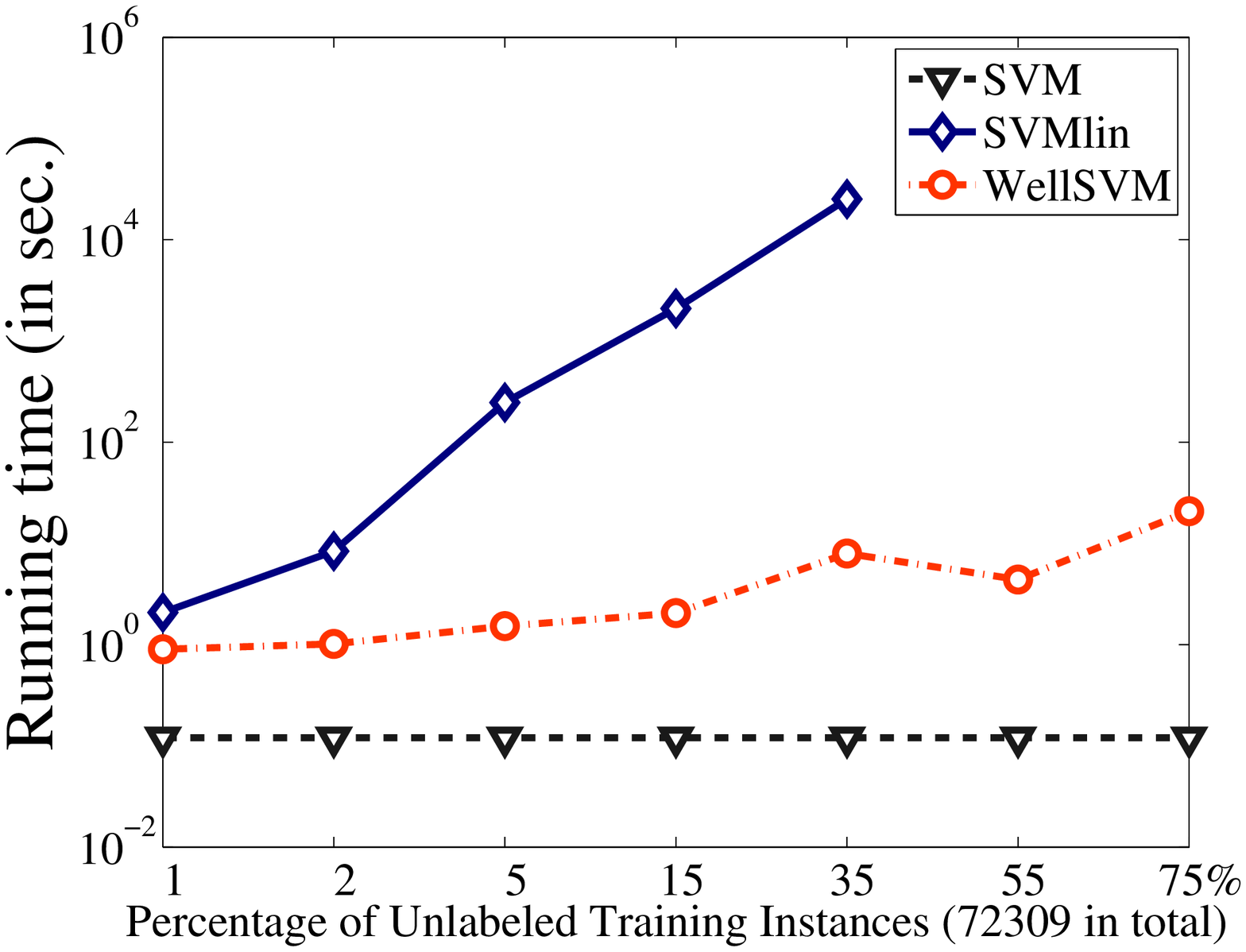}
\end{minipage} \\
\caption{Semi-supervised learning results on the \textit{real-sim} data with different amounts of unlabeled examples.}\label{fig:real-sim-comparison-ssl}\vspace{-8mm}
\end{figure}

\begin{figure}[t]
\centering
\begin{minipage}[c]{2.5 in} \centering \includegraphics[width = 2.5 in]{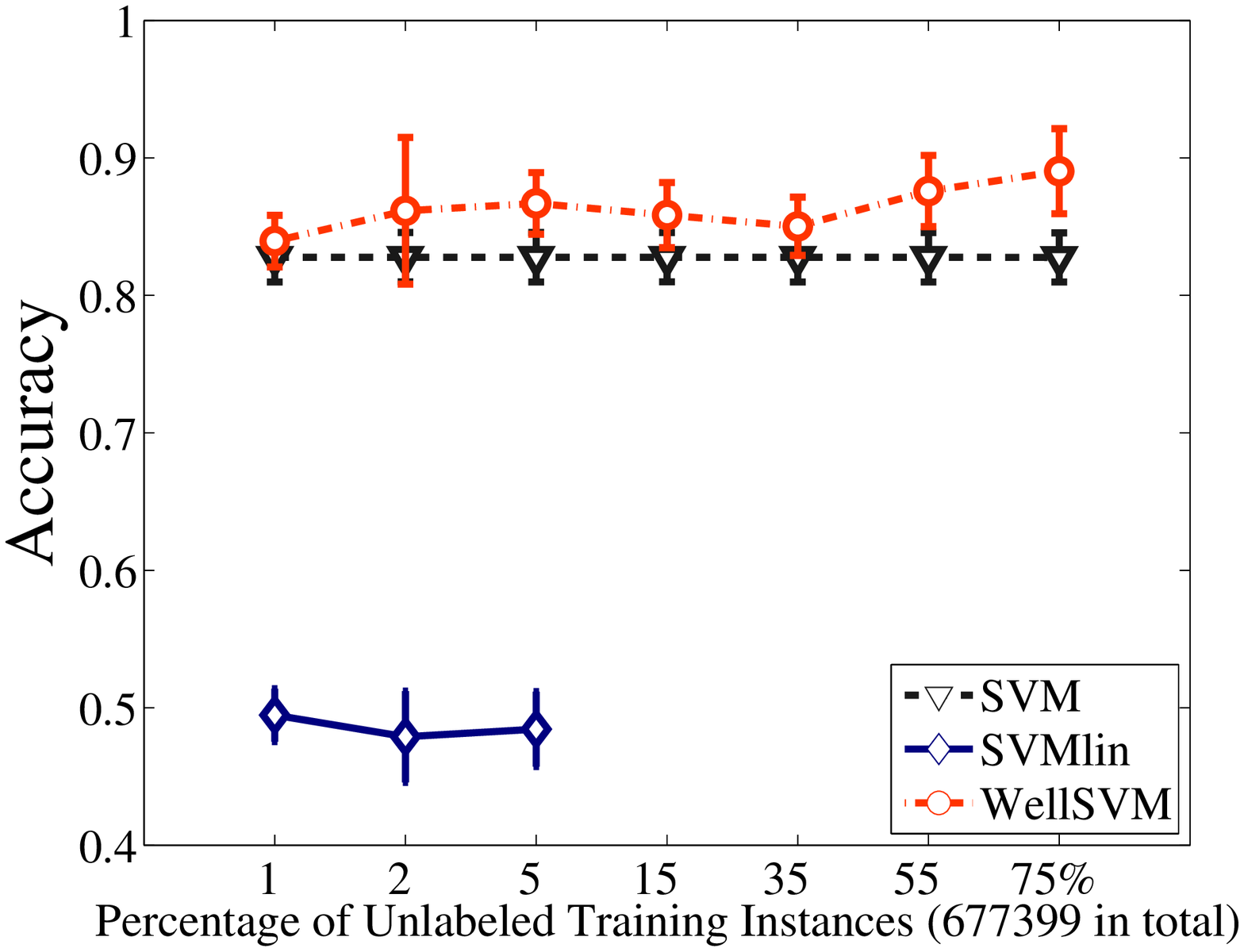}
\end{minipage}
\begin{minipage}[c]{2.5 in} \centering \includegraphics[width = 2.5 in]{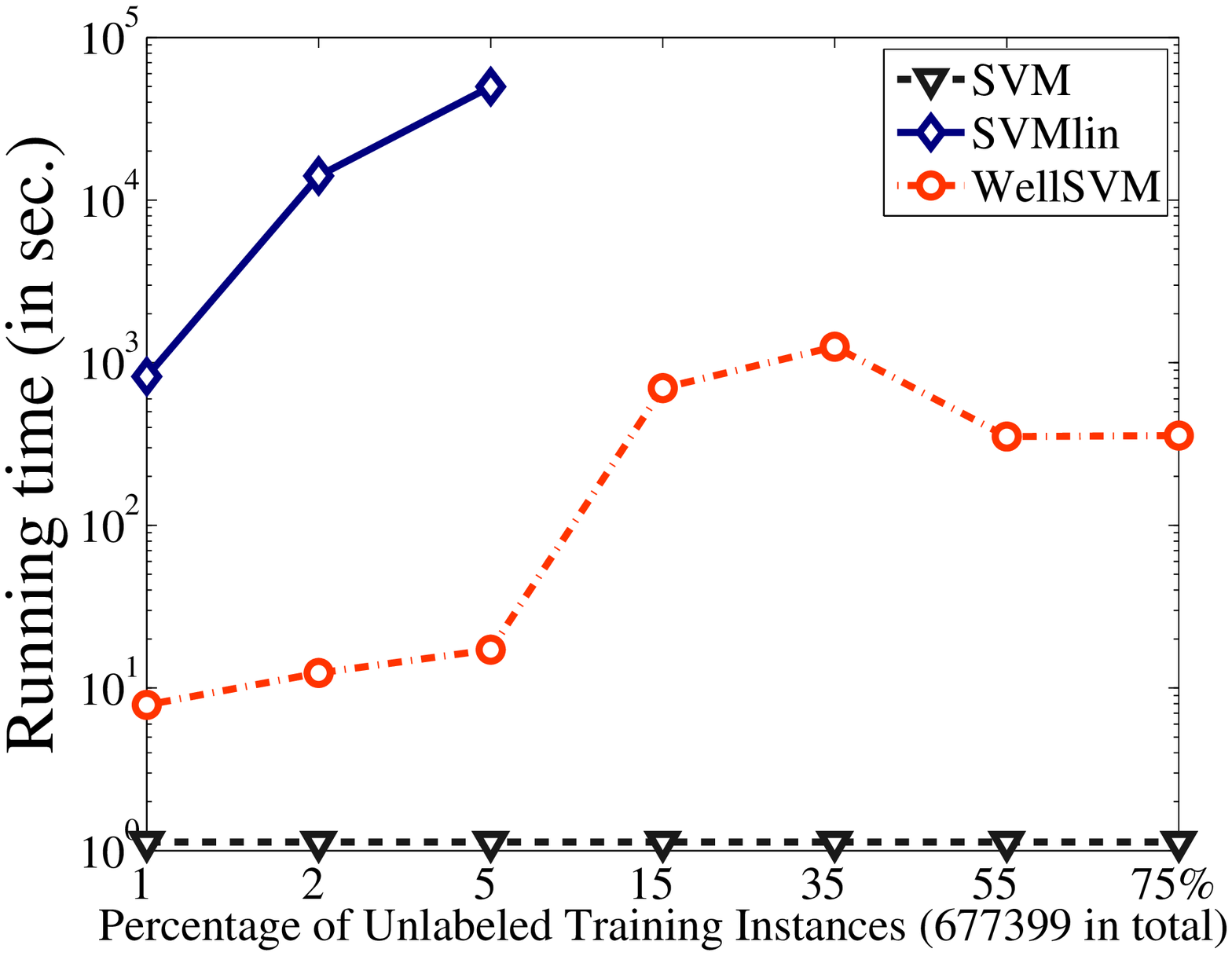}
\end{minipage} \\
\caption{Semi-supervised learning results on the \textit{RCV1} data with different number of unlabeled examples.}\label{fig:rcv1-comparison-ssl}\vspace{-6mm}
\end{figure}

\begin{table}[!h]
\begin{center}\small
\begin{tabular}{c|c|c|c|c|c|c}
\noalign{\smallskip}\hline
\multicolumn{2}{c|}{\# of labeled examples} & 25 & 50 & 100 & 150 & 200 \\
\noalign{\smallskip}\hline
\textit{real-sim} & SVM & 0.78 $\pm$ 0.03 & 0.81 $\pm$ 0.02 & 0.84 $\pm$ 0.02 & 0.86 $\pm$ 0.01 & 0.88 $\pm$ 0.01 \\
 & \lgsvm & \textbf{0.81 $\pm$ 0.08} & \textbf{0.84 $\pm$ 0.02} & \textbf{0.89 $\pm$ 0.01} & \textbf{0.9 $\pm$ 0.01} & \textbf{0.91 $\pm$ 0.01} \\
\noalign{\smallskip}\hline
\textit{rcv1} & SVM & 0.77 $\pm$ 0.03 & 0.83 $\pm$ 0.01 & 0.87 $\pm$ 0.01 & 0.89 $\pm$ 0.01 & 0.9 $\pm$ 0.01 \\
& \lgsvm & \textbf{0.83 $\pm$ 0.03} & \textbf{0.9 $\pm$ 0.02} & \textbf{0.91 $\pm$ 0.01} & \textbf{0.92 $\pm$ 0.01} & \textbf{0.93 $\pm$ 0.01} \\
\noalign{\smallskip}\hline
\end{tabular}
\end{center}\vspace{-4mm}
\caption{Accuracy (with standard derivations) on the \textit{real-sim} and
\textit{rcv1} data sets, with different numbers of labeled examples.
Results for which the performance of \lgsvm\ is \textbf{significantly} better than SVM are in bold.}
\label{table:uci-large} \vspace{-6mm}
\end{table}

Figure~\ref{fig:rcv1-comparison-ssl} shows the results on the larger
\textit{RCV1} data set. As
can be seen, \lgsvm\ obtains good accuracy at different numbers of unlabeled
examples. More importantly, \lgsvm\ scales well on \textit{RCV1}. For example,
\lgsvm\ takes fewer than 1,000 seconds
with more than 500,000 instances.
On the other hand, SVMlin cannot converge in $24$ hours when more than $5\%$
examples are used for training.

Our next experiment studies how the performance of \lgsvm\ changes with
different numbers of labeled examples. Following the
setup in Section~\ref{sec:small}, $75\%$ of the examples are used for training
while the rest are for testing. Different numbers
(namely, 25,50,100,150, and 200)
of labeled examples
are randomly chosen. Since SVMlin cannot handle such a
large training set,
the SVM is used instead.
The above process is repeated $30$ times.
Table~\ref{table:uci-large} shows
the
average testing accuracy.
As can be seen, \lgsvm\ is
significantly better than SVM in all cases. The high
standard
deviation of \lgsvm~on \textit{real-sim} with 25 labeled examples may be due to the
fact that the large amount of unlabeled instances lead to a large variance in deriving a large margin classifier, whereas the amount of labeled examples is too small to reduce the variance.

\subsubsection{Comparison with Other Benchmarks in the Literature}

In this section, we further evaluate the proposed \lgsvm\ with other published results in the literature. First, we experiment on the benchmark data sets in \citet{chapelle2006semi} by using their same setup. Results
on the average test error are shown in Table~\ref{table:Benchmark}. As can be seen, \lgsvm\ is highly competitive.

\vspace{+2mm}
\begin{table}[h]
\begin{center}\small
\begin{tabular}{l|ccccccccccc}
\hline
& g241c & g241d & Digit1 & USPS & COIL & BCI & Text \\
\noalign{\smallskip}\hline
SVM & 47.32 & 46.66 & 30.60 & \textbf{20.03} & 68.36 & 49.85 & 45.37\\
TSVM & \textbf{24.71} & 50.08 & 17.77 & 25.20 & \textbf{67.50} & 49.15 & 40.37 \\
\lgsvm & 37.37 & \textbf{43.33} & \textbf{16.94} & 22.74 & 70.73 & \textbf{48.50} & \textbf{33.70}\\
\noalign{\smallskip}\hline
\end{tabular}
\end{center}\vspace{-4mm}
\caption{Test errors (\%) on the SSL benchmark data sets (using 10 labeled examples).
The SVM and TSVM results are from Table 21.9 in \citet{chapelle2006semi}.}
\label{table:Benchmark}\vspace{-5mm}
\end{table}

\begin{table}[h]
\begin{center}\small
\begin{tabular}{l|c|cccccccc}
\hline
& SVM  & $\nabla$$\text{S}^3$VM & c$\text{S}^3$VM & USVM & TSVM & $\nabla$DA & Newton & BB & \lgsvm \\
\noalign{\smallskip}\hline
2moons & 35.6  & 65.0 & 49.8 & 66.3 & 68.7 & 30.0 & 33.5 & \textbf{0.0} & 33.5 \\
g50c & 8.2 & 8.3 & 8.3 & 8.5 & 8.4 & \textbf{6.7} & 7.5  & -& 7.6 \\
text  & 14.8 & \textbf{5.7} & 5.8 & 8.5 & 8.1 & {6.5} & 14.5  & - & 8.7 \\
uspst & 20.7 & 14.1 & 15.6 & 14.9 & 14.5 & \textbf{11.0} & 19.2 & - & 14.3  \\
coil20 & 32.7 & 23.9 & 23.6 & 23.6 & 21.8 & \textbf{18.9} & 24.6  & - & 23.0  \\
\noalign{\smallskip}\hline
\end{tabular}
\end{center}\vspace{-5mm}
\caption{Test errors (\%) of the \lgsvm\  and various $\text{S}^3$VM variants.
Results of the $\text{S}^3$VMs compared are from Table 11 in \citet{Chapelle2008}.
BB can only be run on the \textit{2moons} data set due to its high computational cost.
Note that in \citet{Chapelle2008}, USVM is called CCCP and TSVM is called $\text{S}^3$VM$^{light}$.}
\label{table:BB}\vspace{-2mm}
\end{table}

Next, we compare \lgsvm\ with the SVM and other state-of-the-art
$\text{S}^3$VMs reported in \citet{Chapelle2008}. These include
\begin{enumerate}
\item $\nabla$$\text{S}^3$VM \citep{chapelle2004semi},
which minimizes the $\text{S}^3$VM objective by
gradient descent;
\item Continuation $\text{S}^3$VM (c$\text{S}^3$VM) \citep{chapelle2006continuation}, which first relaxes the
$\text{S}^3$VM objective to a continuous function and then employs gradient
descent;
\item
USVM \citep{collobert2006lst};
\item TSVM \citep{Joachims1999};
\item Deterministic annealing
$\text{S}^3$VM with gradient minimization ($\nabla$DA) \citep{sindhwani2006das}, which
is based on the global optimization heuristic of deterministic annealing;
\item Newton $\text{S}^3$VM (Newton) \citep{chapelle2007training}, which
uses the second-order Newton's method;  and
\item Branch-and-bound (BB) \citep{chapelle2006bb}.
\end{enumerate}
Results are shown in Table~\ref{table:BB}.
As can be seen, BB attains the best performance. Overall, \lgsvm\ performs slightly
worse than $\nabla$DA, but is highly competitive compared with the other $\text{S}^3$VM variants.

Finally, we compare \lgsvm\ with MMC \citep{xu2005mmc}, a SDP-based S$^3$VM, on the data sets used
there. Table~\ref{table:MMC} shows the results. Again, \lgsvm\ is highly competitive.

\vspace{+2mm}
\begin{table}[h]
\begin{center}\small
\begin{tabular}{l|cccccc}
\hline
& HWD 1-7 & HWD 2-3 & Australian & Flare & Vote & Diabetes \\
\noalign{\smallskip}\hline
MMC & 3.2 & \textbf{4.7} & \textbf{32.0} & 34.0 & 14.0 & \textbf{35.6} \\
\lgsvm & \textbf{2.7} & 5.3 & 40.0 & \textbf{28.9} &  \textbf{11.6} & 41.3 \\
\noalign{\smallskip}\hline
\end{tabular}
\end{center}\vspace{-5mm}
\caption{Test errors (\%) of \lgsvm\ and MMC (a SDP-based S$^3$VM) on the
data sets used in \citet{xu2005mmc}.  The MMC results are copied from their Table~2.}
\label{table:MMC}\vspace{-4mm}
\end{table}

\subsection{Multi-Instance Learning for Locating ROIs}
\label{subsec:mil-expt}

In this section, we evaluate the proposed method on multi-instance learning,
with application to ROI-location in CBIR image data. We employ the image
database in \citet{zhou2005lri},
which consists of 500 COREL images from five image categories:
\textit{castle}, \textit{firework}, \textit{mountain}, \textit{sunset} and
\textit{waterfall}.
Each image is of size $160 \times 160$, and is converted
to the multi-instance feature representation by the bag generator
SBN \citep{maron1998mil}.
Each region (instance) in the image (bag) is of size 20 $\times$ 20. Some of
these regions are labeled manually as ROIs. A summary of the data set is shown
in Table~\ref{table:Statistics}. It is very labor-expensive to collect
large image data with all the regions labeled. Hence, we will leave the
experiments
on large-scale data sets as a future direction.

\vspace{+2mm}
\begin{table}[!ht]
\begin{center}
\begin{tabular}{l|c|c}
\hline
concept & ~\#images~ & ~average \#ROIs per image~ \\ \hline
\textit{castle}  & 100 & 19.39 \\
\textit{firework}  & 100 & 27.23 \\
\textit{mountain}  & 100 & 24.93 \\
\textit{sunset}  & 100 & 2.32 \\
\textit{waterfall}  & 100 & 13.89 \\
\hline
\end{tabular}
\end{center}\vspace{-4mm}
\caption{Some statistics of the image data set.}
\label{table:Statistics}
\end{table}

The one-vs-rest strategy is used. Specifically, a training set of 50 images is created
by randomly sampling 10 images from each of the five categories. The remaining 450
images constitute the test set. This training/test split is randomly generated
10 times, and the average performance is reported.

Although many multi-instance methods have been proposed, they mainly focus on improving the classification performance, whereas only some of them are used to identify the ROIs. We list these state-of-the-art methods \citep{andrews2003svm,maron1998mil,zhang2002dim,zhou2005lri} as well as related SVM-based methods for comparisons in experiments. Specifically,
the \lgsvm\ is compared with the following SVM variants:  1)
MI-SVM \citep{andrews2003svm}; 2)
mi-SVM \citep{andrews2003svm}; and 3) SVM with multi-instance kernel
(MI-Kernel) \citep{gartner2002mik}.
The Gaussian kernel is used for all the SVMs, where its width $\sigma$ is
picked from
$\{0.25\sqrt{\gamma}, 0.5\sqrt{\gamma}, \sqrt{\gamma}, 2\sqrt{\gamma}, 4\sqrt{\gamma}\}$
with $\gamma$ being the average distance between instances;
$C_1$ is picked from
$\{C_2,4 C_2,10 C_2\}$; and
$C_2$ is
from $\{1,10,100\}$.
We also compare with three
state-of-art non-SVM-based
methods
that can
locate
ROIs, namely, Diverse Density (DD)
 \citep{maron1998mil}, EM-DD \citep{zhang2002dim} and
C\textit{k}NN-ROI \citep{zhou2005lri}.
All the parameters are selected by
ten-fold cross-validation (except for C\textit{k}NN-ROI, in which its parameters
are based on the best setting reported in \citet{zhou2005lri}).

In each image classified as relevant by the algorithm, the image region with the maximum prediction value
is taken as its ROI.\footnote{Alternatively, if we allow an algorithm to output multiple
ROI's for an image,
a heuristic thresholding of the prediction values
will be needed. For simplicity, we defer such a setup as future work.}
The following two measures are used in evaluating the performance of ROI location.
\begin{enumerate}
\item
\begin{equation} \label{eq:m1}
\text{success rate of relevant images} = \frac{\text{number of ROI successes} }{
\text{number of relevant images}}.
\end{equation}
Here, for each image predicted as relevant by the algorithm, the ROI returned by the algorithm is counted as a success if it is a real ROI.
\item The ROI success rate computed based on those images that are predicted as relevant, that is,
\begin{equation} \label{eq:m2}
\text{success rate of ROIs} = \frac{\text{number of ROI successes} }{ \text{number of images predicted as relevant}}.
\end{equation}
\end{enumerate}
Notice that there is a tradeoff between these two measures. When
an algorithm classifies many images as relevant, the success rate of relevant
images (Equation~(\ref{eq:m1})) is
high while the success rate of ROIs (Equation~(\ref{eq:m2})) can be low, since there are many relevant
images predicted by the algorithm. On the other hand, when an algorithm classifies many images as irrelevant, the success rate of ROIs is high while the success rate of relevant images is low since many relevant images are missing. To compromise
these two goals,
we introduce a novel \textit{success rate} of ROIs
\[
\textit{success rate} = \frac{2\# \text{ROI successes}}{\# \text{relevant images}+\# \text{predicted relevant images}}.
\]
This is similar to the F-score in information retrieval as
\begin{eqnarray*}
\frac{1}{\textit{success rate}} &=& \frac{ \# \text{relevant images}+\# \text{predicted relevant images} }{2\# \text{ROI successes}} \nonumber\\
&=& \frac{1}{2} \left( \frac{1}{\frac{\# \text{ROI successes} }{ \# \text{relevant images}}}+ \frac{1}{ \frac{\# \text{ROI successes} }{ \# \text{predicted relevant images}} } \right).
\end{eqnarray*}
Intuitively, when an algorithm correctly recognizes all the relevant images and
their ROIs, the success rate will be high.

Table~\ref{table:ROI} shows the success rates (with standard deviations) of the various methods. As can be seen, \lgsvm\ achieves the best performance among all the SVM-based methods. As for its performance comparison with the other non-SVM methods, \lgsvm\ is still always better than DD and C$k$NN-ROI, and is highly comparable to EM-DD.
In particular, EM-DD achieves the best performance on \textit{castle} and \textit{sunset}, while \lgsvm\ achieves the best performance on the remaining three categories (\textit{firework}, \textit{mountain} and \textit{waterfall}). Figure~\ref{fig:roi} shows some example images with the located ROIs. It can be observed that \lgsvm\ can correctly identify more ROIs than the other SVM-based methods.

\begin{table}[t]
\begin{center}\small
\begin{tabular}{cc|c@{\!}c@{\!}c@{\!}c@{\!}c@{\!}} \hline
& method & \textit{castle} & \textit{firework} & \textit{mountain} & \textit{sunset} & \textit{waterfall} \\
\hline\noalign{\smallskip}
& \lgsvm  &~~{0.57 $\pm$ 0.12}~~&~~\textbf{0.68 $\pm$ 0.17}~~  &~~\textbf{0.59 $\pm$ 0.10}~~  &~~{0.32 $\pm$ 0.07}~~ &~~\textbf{0.39 $\pm$ 0.13 }~~  \\
\cline{2-7}\noalign{\smallskip}
SVM & mi-SVM             & 0.51 $\pm$ 0.04  &0.56 $\pm$ 0.07  &{0.18 $\pm$ 0.09}  &0.32 $\pm$ 0.01 & 0.37 $\pm$ 0.08   \\
\cline{2-7}\noalign{\smallskip}
methods & MI-SVM   &{0.52 $\pm$ 0.22}  &{ 0.63 $\pm$ 0.26}  & 0.18 $\pm$ 0.13  & 0.29 $\pm$ 0.10  & 0.06 $\pm$ 0.02   \\
\cline{2-7}\noalign{\smallskip} & MI-Kernel &0.56 $\pm$ 0.08  & 0.57 $\pm$ 0.11  & 0.23 $\pm$ 0.20  & 0.24 $\pm$ 0.03 & 0.20 $\pm$ 0.11   \\
\hline\hline\noalign{\smallskip}
& DD   &0.24 $\pm$ 0.16  & 0.15 $\pm$ 0.28  & 0.56 $\pm$ 0.11  & 0.30 $\pm$ 0.18 & 0.26 $\pm$ 0.24   \\
\cline{2-7}\noalign{\smallskip}
non-SVM & EM-DD   &\textbf{0.69 $\pm$ 0.06}  & 0.65 $\pm$ 0.24 &{0.54 $\pm$ 0.18 } & \textbf{0.36 $\pm$ 0.15}  &{0.30 $\pm$ 0.12}   \\
\cline{2-7}\noalign{\smallskip}
methods & C\textit{k}NN-ROI  &0.48 $\pm$ 0.05  & 0.65 $\pm$ 0.09  & 0.47 $\pm$ 0.06  & 0.31 $\pm$ 0.04 &0.20 $\pm$ 0.05   \\
\hline
\end{tabular}
\end{center} \vspace{-4mm}
\caption{Success rate in locating the ROIs. The best performance and those which are comparable to the best performance (paired $t$-test at $95\%$ significance level) on each data set are bolded.}\label{table:ROI} \vspace{-4mm}
\end{table}

\begin{figure}[t]
\centering
\includegraphics[height=2.2in]{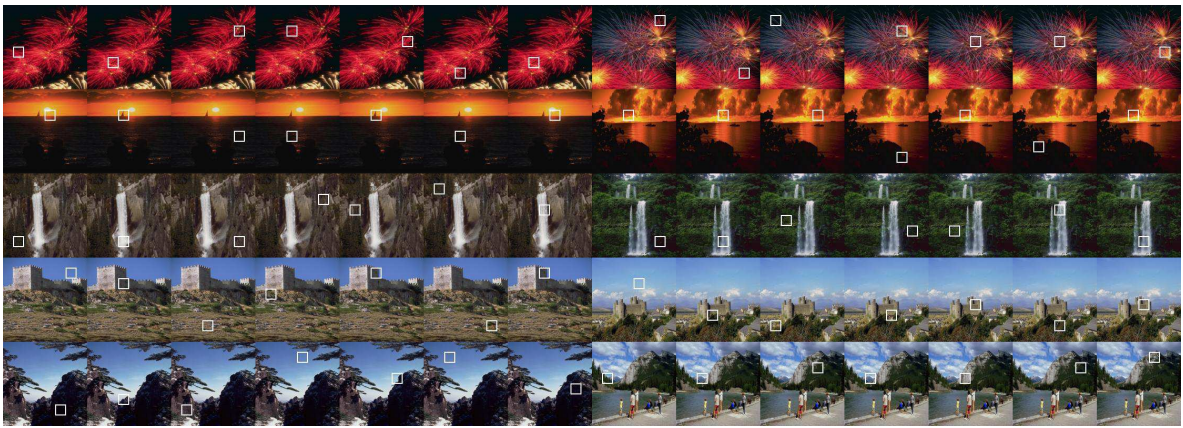}\vspace{-0mm}
\centering \caption{ROIs located by (from left to right) DD, EM-DD, C\textit{k}NN-ROI, MI-SVM, mi-SVM, MI-Kernel,
and \lgsvm. Each row shows one category (top to bottom: \textit{firework}, \textit{sunset}, \textit{waterfall},
\textit{castle} and \textit{mountain}).} \label{fig:roi} \vspace{-6mm}
\end{figure}

\subsection{Clustering}
\label{subsec:clustering-expt}

In this section, we further evaluate our \lgsvm\ on clustering problems where all
the labels are unknown. As in semi-supervised learning, 16 UCI data sets and 2 large data sets are used for comparison.

\subsubsection{Small-Scale Experiments}

The \lgsvm\ is compared with the following methods: 1) $k$-means clustering (KM); 2) kernel $k$-means clustering (KKM); 3) normalized cut (NC) \citep{shi2000nca}; 4) GMMC \citep{valizadegan1400gmm}; 5) IterSVR\footnote{IterSVR can be found at \url{http://www.cse.ust.hk/~twinsen}.} \citep{zhang2007mmc}; and 6) CPMMC\footnote{CPMMC can be found at \url{http://binzhao02.googlepages.com/}.} \citep{zhao:emm}. In the preliminary experiment, we also compared with the original SDP-based approach in \citet{xu2005mmc}. However, similar to the experimental results in semi-supervised learning, it does not converge after $3$ hours on the smallest data set \textit{echocardiogram}. Hence, GMMC, which is also based on SDP but about 100 times faster than \citet{xu2005mmc}, is used in the
comparison.

For GMMC, IterSVR, CPMMC and \lgsvm, the $C$ parameter is selected in a range \linebreak[4] $\{0.1,0.5,1,5,10,100\}$. For the UCI data sets, both the linear and Gaussian kernels are used. In particular, the width $\sigma$ of the Gaussian kernel is picked from $\{0.25\sqrt{\gamma},0.5\sqrt{\gamma},\sqrt{\gamma},$ $2\sqrt{\gamma},4\sqrt{\gamma}\}$, where $\gamma$ is the average distance between instances. The parameter of normalized cut is chosen from the same
range of $\sigma$. Since $k$-means and IterSVR are susceptible to the problem of local minimum, these two methods are run 10 times and the average performance reported. We set the balance constraint in the same manner as in \citet{zhang2007mmc}, that is, $\beta$ is set as $0.03N$ for balanced data and $0.3N$ for imbalanced data. To initialize \lgsvm, 20 random label assignments are generated and the one with the maximum kernel alignment \citep{cristianini2002kernel} is chosen. We also use this to initialize KM, KKM and IterSVR, and the resultant variants are denoted KM-r, KKM-r and IterSVR-r, respectively. All the methods are reported with the best parameter setting.

We follow the strategy in \citet{xu2005mmc} to evaluate the clustering accuracy. We first remove the labels for all instances, and then obtain the clusters by the various clustering algorithms. Finally, the misclassification error is measured w.r.t. the true labels.

We first study the clustering accuracy on 16 UCI data sets that cover a wide range of properties. Results are shown in
Table~\ref{table:uci-mmc}. As can be seen, \lgsvm\ outperforms existing clustering approaches on most data sets. Specifically, \lgsvm\
obtains the best performance on 10 out of 16 data sets. GMMC is not as good as \lgsvm. This may due to that the convex relaxation
proposed in GMMC is not the same as the original SDP-based approach \citep{xu2005mmc} and \lgsvm.

\setlength{\tabcolsep}{5pt}
\begin{table}[t]\small
\begin{center}
\begin{tabular}{lcccccccccc}
\hline\noalign{\smallskip}
& & & & & & & Iter & Iter & CP & \textsc{Well}\\
Data & KM & KM-r & KKM & KKM-r & NC  & GMMC & SVR & SVR-r & MMC & \textsc{SVM}  \\
\hline\noalign{\smallskip}
\textit{Echocardiogram} & 0.76 & 0.76 & 0.76 & 0.77 & 0.76 & 0.7 & 0.74 & 0.78 & 0.82 & \textbf{0.84} \\
\textit{House} & 0.89 & 0.89 & 0.89 & 0.88 & 0.89 & 0.78 & 0.87 & 0.87 & 0.53 & \textbf{0.90} \\
\textit{Heart} & 0.66 & 0.59 & 0.69 & 0.59 & 0.57 & \textbf{0.7} & 0.59 & 0.59 & 0.56 & {0.59} \\
\textit{Heart-statlog} & 0.68 & 0.79 & 0.78 & 0.79 & 0.79 & 0.77 & 0.76 & 0.76 & 0.56 & \textbf{0.81} \\
\textit{Haberman} & 0.6 & 0.59 & 0.69 & 0.64 & 0.7 & 0.6 & 0.62 & 0.57 & \textbf{0.74} & \textbf{0.74} \\
\textit{LiverDisorders} & 0.55 & 0.54 & 0.56 & 0.56 & 0.57 & 0.55 & 0.53 & 0.51 & \textbf{0.58} & \textbf{0.58} \\
\textit{Spectf} & 0.58 & 0.57 & \textbf{0.77} & \textbf{0.77} & 0.63 & 0.64 & 0.53 & 0.53 & 0.73 & 0.73 \\
\textit{Ionosphere} & 0.7 & 0.71 & 0.73 & \textbf{0.74} & 0.7 & 0.73 & 0.71 & 0.65 & 0.64 & {0.72} \\
\textit{House-votes} & \textbf{0.87} & \textbf{0.87} & \textbf{0.87} & \textbf{0.87} & 0.86 & 0.6 & 0.83 & 0.82 & 0.61 & \textbf{0.87} \\
\textit{Clean1} & 0.54 & 0.54 & 0.59 & 0.62 & 0.52 & \textbf{0.66} & 0.61 & 0.53 & 0.56 & 0.55 \\
\textit{Isolet} & 0.98 & 0.96 & 0.89 & 0.95 & 0.98 & 0.56 & \textbf{1.00} & \textbf{1.00} & 0.5 & {0.98} \\
\textit{Australian} & 0.54 & 0.55 & 0.57 & 0.57 & 0.56 & 0.6 & 0.56 & 0.51 & 0.56 & \textbf{0.83} \\
\textit{Diabetes} & 0.67 & 0.67 & \textbf{0.69} & \textbf{0.69} & 0.66 & \textbf{0.69} & 0.66 & 0.66 & 0.65 & \textbf{0.69} \\
\textit{German} & 0.57 & 0.56 & 0.68 & 0.62 & 0.66 & 0.56 & 0.56 & 0.64 & \textbf{0.7} & \textbf{0.7} \\
\textit{Krvskp} & 0.52 & 0.51 & 0.55 & 0.55 & \textbf{0.56} &  -  & 0.51 & 0.51 & 0.52 & {0.54} \\
\textit{Sick} & 0.68 & 0.63 & 0.88 & 0.77 & 0.84 &  -  & 0.63 & 0.59 & \textbf{0.94} & \textbf{0.94} \\
\noalign{\smallskip}\hline
\end{tabular}
\end{center}\vspace{-4mm}
\caption{Clustering accuracies on various data sets.
``-'' indicates that the method does not converge in 2 hours or out-of-memory problem occurs.
}\label{table:uci-mmc}\vspace{-5mm}
\end{table}

The CPU time on the UCI data sets are shown in
Figure~\ref{fig:cputimes-UCI-clustering}. As can be seen, local optimization
methods, such as IterSVR and CPMMC, are often efficient.
As for the global optimization method, \lgsvm\ scales much better than GMMC. On average, \lgsvm\ is about 10
times faster. These results validate that \lgsvm\ achieves much better scalability than the SDP-based GMMC approach. However, in general, convex methods are still slower than non-convex optimization methods on the small data sets.

\begin{figure}[t]
\centering
\includegraphics[height=2 in]{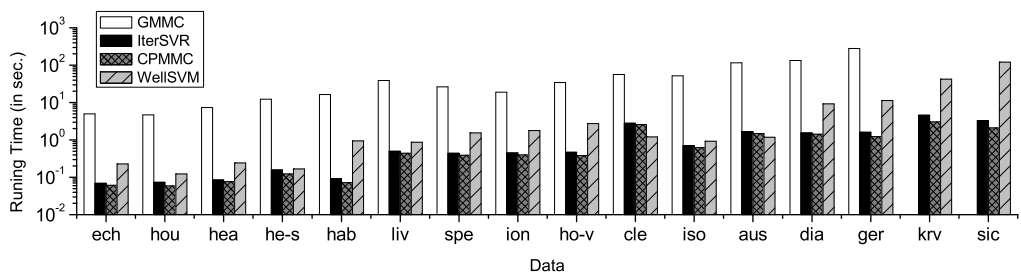}\vspace{-10mm}
\caption{CPU time (in seconds) on the UCI data sets.}\label{fig:cputimes-UCI-clustering}\vspace{-1mm}
\end{figure}

\subsubsection{Large-Scale Experiments}

In this section, we further evaluate the scalability of \lgsvm\ on large data sets
when the linear kernel is used.  In this case, the \lgsvm\ only involves solving a sequence of
linear SVMs. As packages specially designed for the linear SVM (such as
LIBLINEAR) are  much more efficient than those designed  for general
kernels (such as LIBSVM), it can be expected that the linear \lgsvm\ is also
scalable
on large data sets.

The \textit{real-sim} data contains 72,309 instances and has 20,958 features. To study the effect of sample size on performance, different
sampling rates ($1\%$, $2\%$, $5\%$ and $10\%, 20\%, \ldots,100\%$) are considered. For each sampling rate (except for $100\%$), we
perform random sampling 5 times, and report the average performance. Since $k$-means depends on random initialization,
we run it 10 times for each sampling rate, and report its average accuracy. Figure~\ref{fig:real-sim-comparison} shows
the accuracy
and running time.\footnote{$k$-means is implemented in matlab, and so
its running time
is not compared
with \lgsvm,
whose core procedure
is implemented in C++.} As can be seen, \lgsvm\ outperforms $k$-means and can be used on large data sets.

\begin{figure}[t]
\centering
\begin{minipage}[c]{2.5 in} \centering \includegraphics[width = 2.5 in]{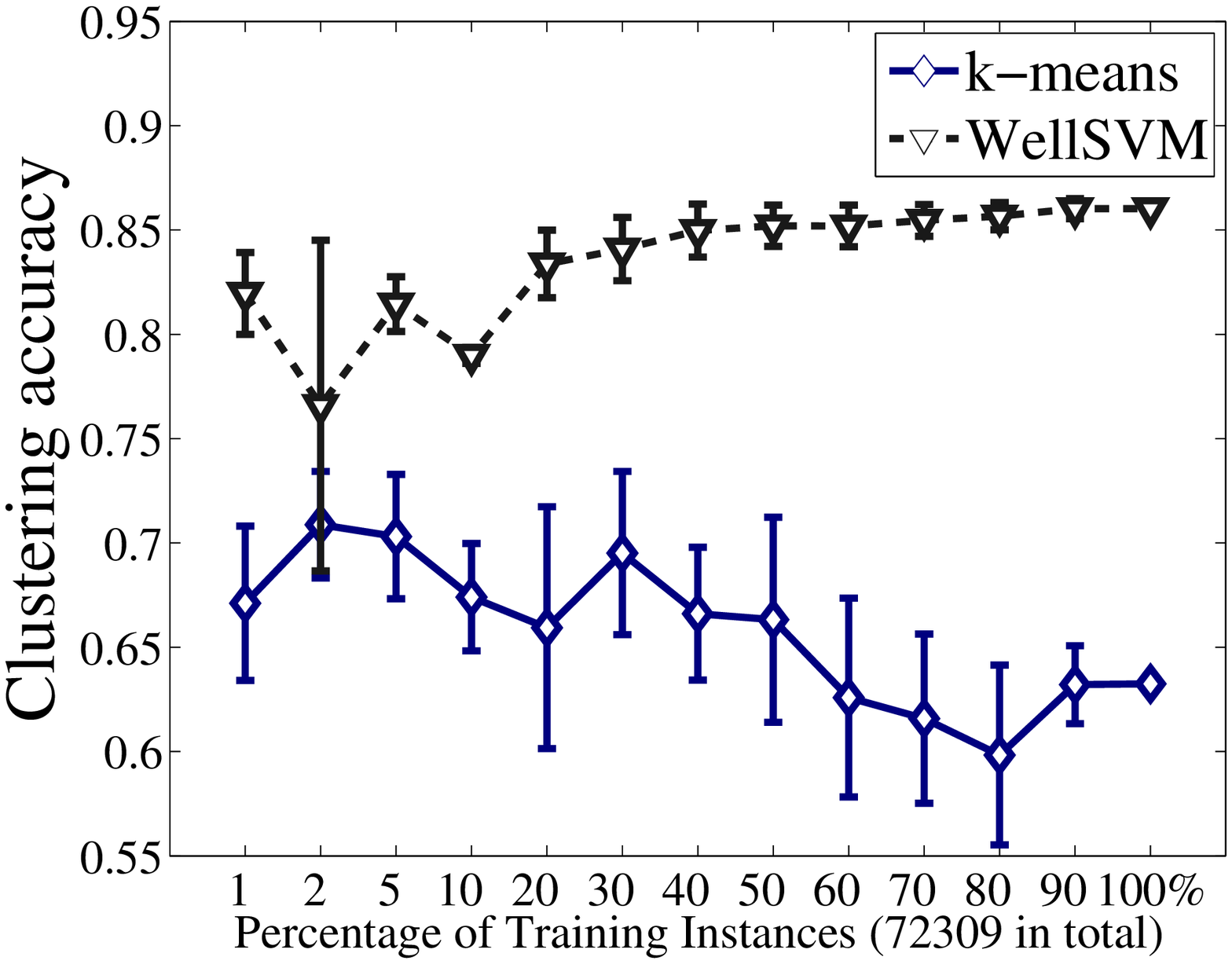}
\end{minipage}\;\;\;\;
\begin{minipage}[c]{2.5 in} \centering \includegraphics[width = 2.5 in]{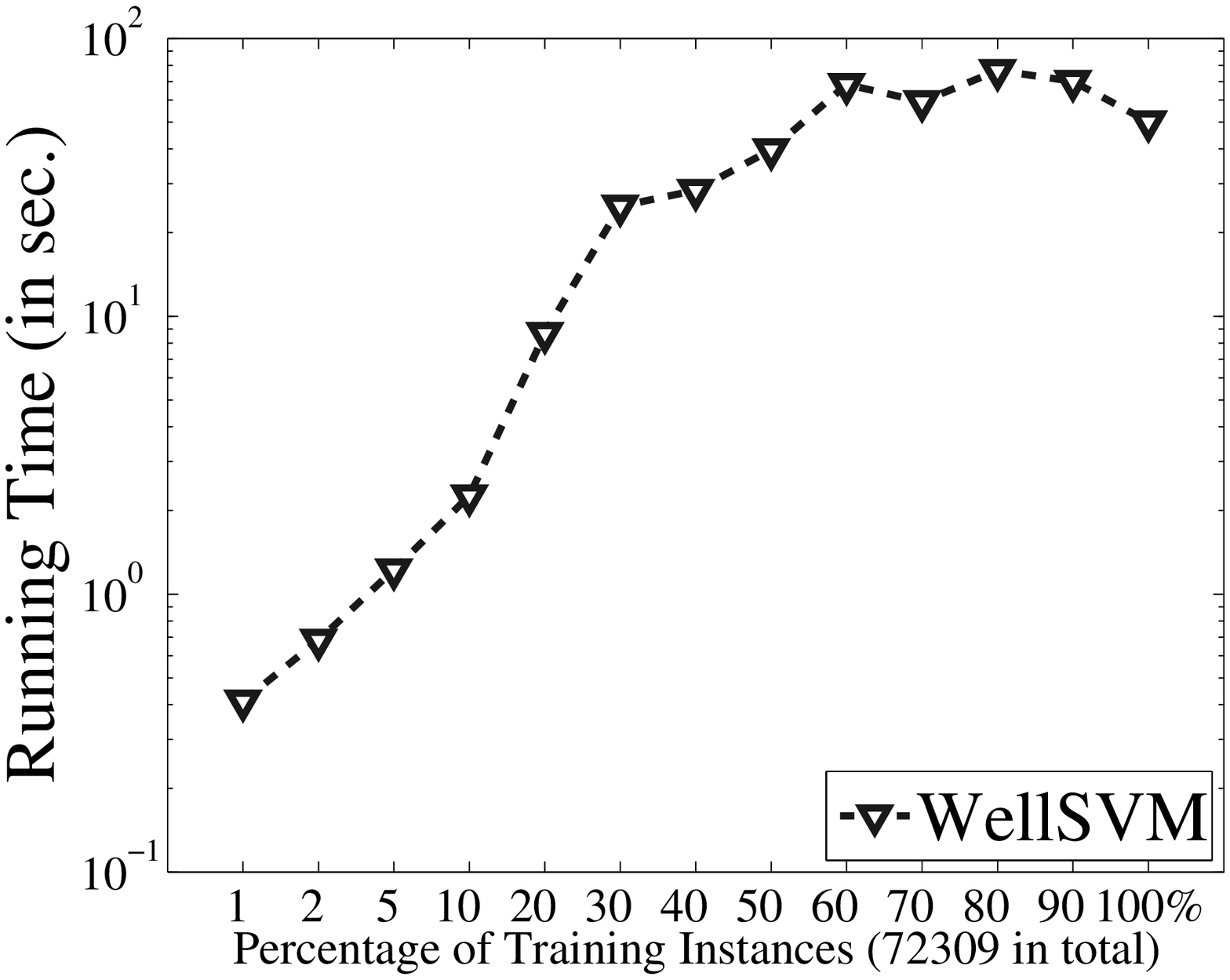}
\end{minipage} \\
\caption{Clustering results on the \textit{real-sim} data with different numbers of examples.}\label{fig:real-sim-comparison}
\end{figure}

\begin{figure}[t]
\centering
\begin{minipage}[c]{2.5 in} \centering \includegraphics[width = 2.5 in]{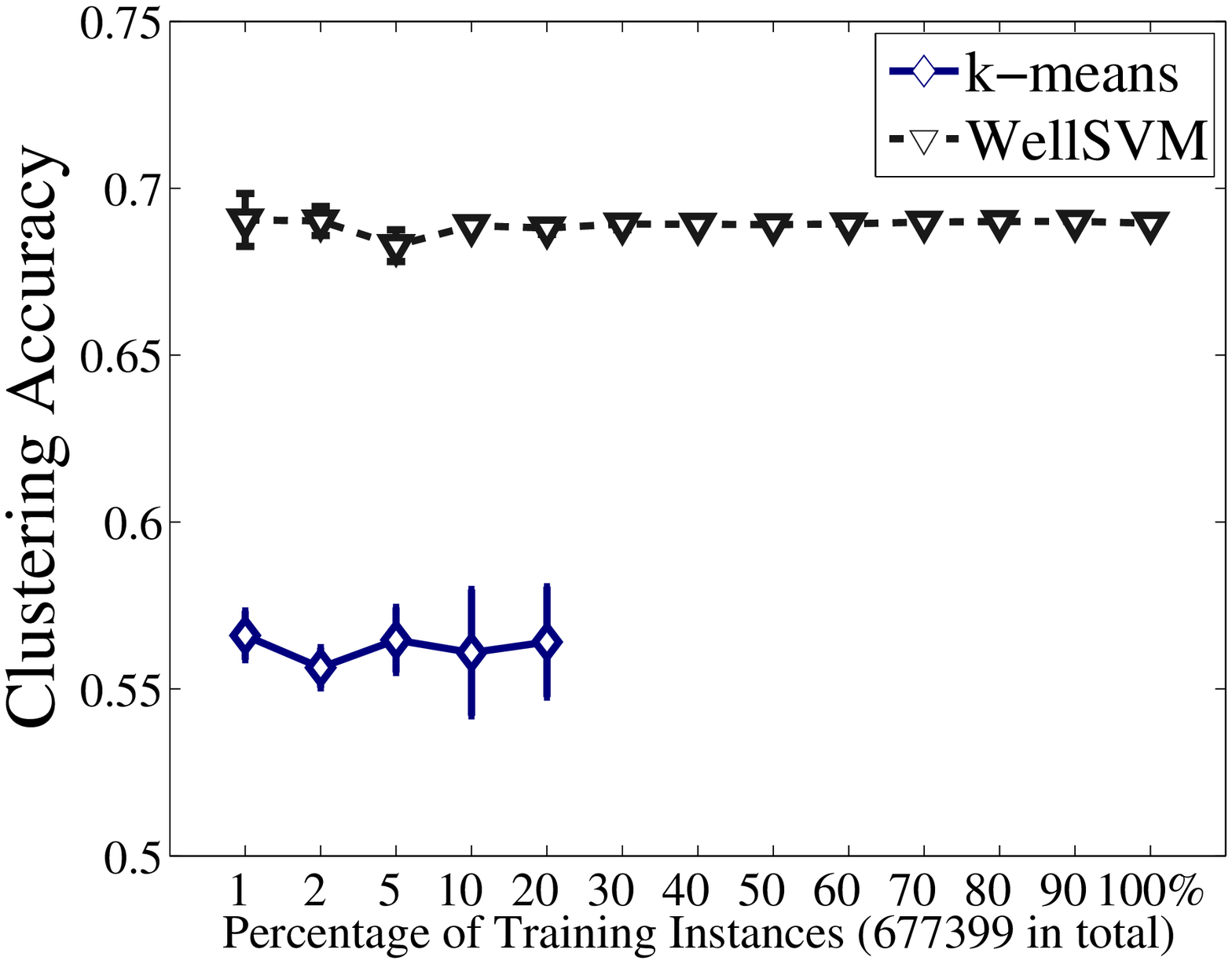}
\end{minipage}\;\;\;\;
\begin{minipage}[c]{2.5 in} \centering \includegraphics[width = 2.5 in]{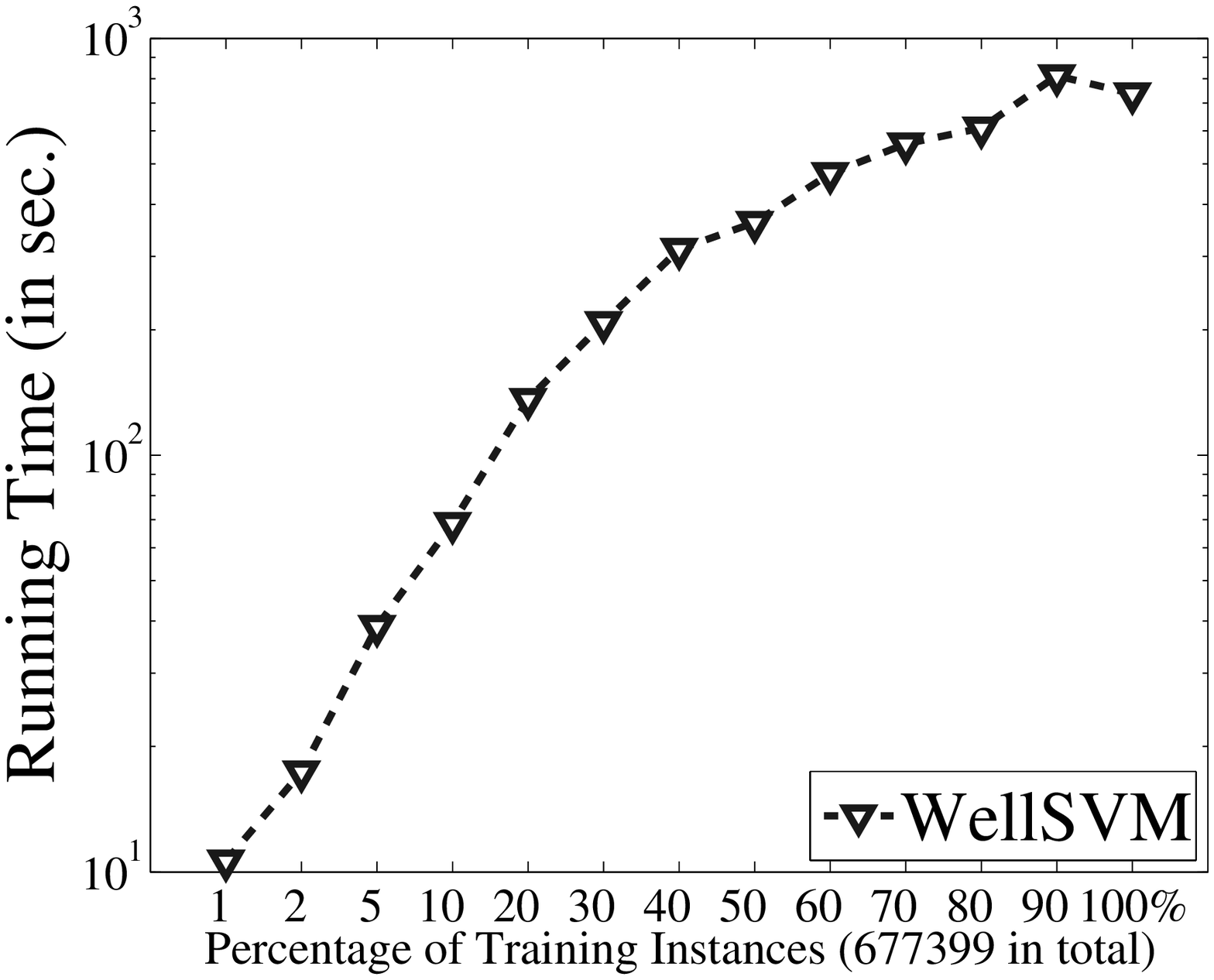}
\end{minipage} \\
\caption{Clustering results on the \textit{RCV1} data with different numbers of examples.}\label{fig:rcv1-comparison}
\end{figure}

The \textit{RCV1} data is very high-dimensional and contains more than 677,000 instances. Following the same setup as for the \textit{Real-sim} data, \lgsvm\ is compared with $k$-means under different sampling rates.
Figure~\ref{fig:rcv1-comparison} shows the results. Note that $k$-means does not converge in 24 hours when more than 20\% training instances are used. As can be seen, \lgsvm\ obtains better performance than $k$-means and \lgsvm\ scales quite well on \textit{RCV1}. It takes fewer than 1,000 seconds for \textit{RCV1} with more than 677,000 instances and 40,000 dimensions.

\section{Conclusion}

Learning from \emph{weakly labeled data}, where the training labels are incomplete, is generally regarded as a crucial yet challenging machine
learning task. However, because of the underlying mixed integer programming
problem, this limits its scalability and accuracy.
To alleviate these difficulties, we proposed
a convex \lgsvm\ based on a novel ``label
generation" strategy. It can be shown that \lgsvm\ is
at least as tight as existing SDP relaxations, but
is much more scalable.
Moreover, since it can be reduced to a sequence of standard SVM training, it can
directly benefit from advances in the development of efficient SVM software.

In contrast to traditional approaches that are tailored for a specific weak-label
learning problem, our \lgsvm\ formulation can be used on a general class of weak-label learning problems.
Specifically, \lgsvm\ on three common weak-label learning tasks, namely (i)
semi-supervised learning where labels are partially known; (ii) multi-instance
learning where labels are implicitly known; and (iii) clustering where labels are
totally unknown, can all be put under the same formulation.
Experimental results show that the \lgsvm\ obtains good performance and is readily
scalable on large data sets. We believe that
similar conclusions can be reached on other weak-label learning tasks, such as
the noisy-tolerant problem \citep{angluin1988learning}.

The focus of this paper is on binary weakly labeled problems.
For multi-class weakly labeled problems, they can be easily handled by
decomposing into multiple binary problems \citep{crammer2002algorithmic}.
However, one
exception is clustering problems, in which existing decomposition methods cannot be applied
as there is no label.
Extension to this more challenging multi-class clustering scenario will
be considered as a future work.

\acks{The authors want to thank the editor and reviewers for helpful comments and
suggestions. We also thank Teng Zhang, Linli Xu and Kai Zhang for help in the experiments.
This work was partially supported by the National Fundamental Research Program of China
(2010CB327903), the National Science Foundation of China (61073097, 61021062), the program
for outstanding PhD candidate of Nanjing University, Singapore A*star under Grant SERC 112
280 4005, and the Research Grants Council
of the Hong Kong Special Administrative Region under Grant 614012. Z.-H. Zhou is the corresponding author of this paper.}

\appendix

\section{Proof of Theorem \ref{them:objective_gain}}
\label{app:proof of theorem}

\begin{proof}
Let $\{\bar{\ba}^{(t)},\bar{\bmu}^{(t)}\}$ be the optimal solution of
Equation~(\ref{eq:LGSVM_framework_v22}), which
can be viewed as a saddle-point problem. Let $J({\ba},{\bmu}) = \sum_{\hat{\y} \in
\C^{(t)} }{\mu}_{\hat{\y}} g_{\hat{\y}}({\ba})$. Using
the saddle-point property \citep{boyd2004co}, we have
\[
J(\ba,\bar{\bmu}^{(t)}) \geq J(\bar{\ba}^{(t)},\bar{\bmu}^{(t)}) \geq
J(\bar{\ba}^{(t)},{\bmu}), \;\; \forall \ba, \bmu.
\]
In other words, $\bar{\ba}^{(t)}$ minimizes $J(\ba,\bar{\bmu}^{(t)})$. Note that
$g_{{\y}}(\ba)$ is $\lambda$-strongly convex
and $\sum_{\y \in \C^{(t)}} \bar{\mu}_{\y}^{(t)} = 1$, thus $J({\ba},\bar{\bmu}^{(t)})$
is also $\lambda$-strongly convex. Using the Taylor expansion,
we have
\[
J({\ba},\bar{\bmu}^{(t)}) - J(\bar{\ba}^{(t)},\bar{\bmu}^{(t)}) \geq \frac{\lambda}{2}
\|\ba-\bar{\ba}^{(t)}\|^2, \;\;
\forall \ba \in \A.
\]
Using the definition of $J({\ba},{\bmu})$, we then have
\begin{equation}\label{eq:ineq1}
\sum_{\hat{\y} \in \C^{(t)}}\bar{\mu}_{\hat{\y}}^{(t)}g_{\hat{\y}}(\ba) - \sum_{\hat{\y} \in
\C^{(t)}}\bar{\mu}_{\hat{\y}}^{(t)}g_{\hat{\y}}(\bar{\ba}^{(t)}) \geq \frac{\lambda}{2} \sum_{\hat{\y} \in \C^{(t)}}
\bar{\mu}_{\hat{\y}}^{(t)}\|\ba-\bar{\ba}^{(t)}\|^2 = \frac{\lambda}{2} \|\ba-\bar{\ba}^{(t)}\|^2.
\end{equation}

Let $\hat{\y}^{(t+1)}$ be the violated label vector selected at iteration $t+1$ in Algorithm \ref{alg:LGSVM}, that is, $\C^{t+1} = \C^{(t)}
\bigcup \hat{\y}^{(t+1)}$. From the definition, we have
\begin{eqnarray}\label{eq:ineq2}
g_{\hat{\y}^{(t+1)}}(\bar{\ba}^{(t)}) = -G(\bar{\ba}^{(t)},\hat{\y}^{(t+1)}) &\geq& \max_{\hat{\y} \in \C^{(t)}}-G(\bar{\ba}^{(t)},\hat{\y})+\eps
= \max_{\hat{\y} \in \C^{(t)}}g_{\hat{\y}}(\bar{\ba}^{(t)})+\eps \nonumber \\
&\geq& \sum_{\hat{\y} \in \C^{(t)}}\bar{\mu}_{\hat{\y}}^{(t)}g_{\hat{\y}}(\bar{\ba}^{(t)}) + \eps = -p^{(t)}+\eps.
\end{eqnarray}
Consider the following optimization problem and let $\hat{p}^{(t+1)}$ be its optimal objective value
\begin{equation}\label{eq:reduced-form}
\hat{p}^{(t+1)} = -\min_{\ba \in \A} \max_{0 \leq \theta \leq 1} \theta\sum_{\hat{\y} \in \C^{(t)}} \bar{\mu}_{\hat{\y}}^{(t)}g_{\hat{\y}}(\ba)+
(1-\theta) g_{\hat{\y}^{(t+1)}}(\ba).
\end{equation}
When $\theta = 1$, it reduces to Equation~(\ref{eq:LGSVM_framework_v22}) at
iteration $t$, and so $\hat{p}^{(t+1)} \leq p^{(t)}$. On the other hand,
note that $\theta\sum_{\hat{\y} \in \C^{(t)}}\bar{\mu}_{\hat{\y}}^{(t)} + (1-\theta) = \theta+(1-\theta) =1$, the optimal solution in
Equation~(\ref{eq:reduced-form}) is suboptimal to that of Equation~(\ref{eq:LGSVM_framework_v22}) at iteration $t+1$. Then we have $p^{(t+1)} \leq
\hat{p}^{(t+1)}$. Let $\hat{p}^{(t+1)} = p^{(t)} - \eta$, now we aims at showing $\eta \geq (\frac{-c+\sqrt{c^2+4\epsilon}}{2})^2$ which obviously
induces our final inequality Equation~(\ref{eq:objective_gain}).

Let $\{\tilde{\ba}^{(t)},\tilde{\theta}\}$ be the optimal solution of Equation~(\ref{eq:reduced-form}), we have following inequalities
\begin{eqnarray}
p^{(t)}-\eta &\leq & -\sum_{\hat{\y} \in \C^{(t)}}
\bar{\mu}_{\hat{\y}}^{(t)}g_{\hat{\y}}(\tilde{\ba}^{(t)}), \label{eq:ineq3}\\
p^{(t)}-\eta &\leq & -g_{\hat{\y}^{(t+1)}}(\tilde{\ba}^{(t)}). \label{eq:ineq4}
\end{eqnarray}
Using Equations~(\ref{eq:ineq1}), (\ref{eq:ineq2}), (\ref{eq:ineq3}) and (\ref{eq:ineq4}), we have
\begin{eqnarray}
\eta & \geq & \sum_{\hat{\y} \in \C^{(t)}}
\bar{\mu}_{\hat{\y}}^{(t)}g_{\hat{\y}}(\tilde{\ba}^{(t)}) - \sum_{\hat{\y} \in
\C^{(t)}}\bar{\mu}_{\hat{\y}}^{(t)}g_{\hat{\y}}(\bar{\ba}^{(t)})  \geq
\frac{\lambda}{2}\|\tilde{\ba}^{(t)}-\bar{\ba}^{(t)}\|^2, \label{eq:ineq5}\\
\eps-\eta &\leq & g_{\hat{\y}^{(t+1)}}(\bar{\ba}^{(t)})-g_{\hat{\y}^{(t+1)}}(\tilde{\ba}^{(t)}) \leq M
\|\tilde{\ba}^{(t)}-\bar{\ba}^{(t)}\|. \label{eq:ineq6}
\end{eqnarray}
On combining Equations~(\ref{eq:ineq5}) and (\ref{eq:ineq6}), we obtain
\[
\epsilon - \eta \leq M \sqrt{\frac{2\eta}{\lambda}},
\]
and then finally we have
${\eta} \geq \Big( \frac{-c+\sqrt{c^2+4\epsilon}}{2} \Big)^2$, where $c = \frac{M}{\sqrt{\lambda/2}}$.
\end{proof}

\bibliography{lgsvm}

\begin{thebibliography}{73}
\providecommand{\natexlab}[1]{#1}
\providecommand{\url}[1]{\texttt{#1}}
\expandafter\ifx\csname urlstyle\endcsname\relax
  \providecommand{\doi}[1]{doi: #1}\else
  \providecommand{\doi}{doi: \begingroup \urlstyle{rm}\Url}\fi

\bibitem[Andrews et~al.(2003)Andrews, Tsochantaridis, and
  Hofmann]{andrews2003svm}
S.~Andrews, I.~Tsochantaridis, and T.~Hofmann.
\newblock Support vector machines for multiple-instance learning.
\newblock In S.~Becker, S.~Thrun, and K.~Obermayer, editors, \emph{Advances in
  Neural Information Processing Systems 15}, pages 577--584. MIT Press,
  Cambridge, MA, 2003.

\bibitem[Angluin and Laird(1988)]{angluin1988learning}
D.~Angluin and P.~Laird.
\newblock Learning from noisy examples.
\newblock \emph{Machine Learning}, 2\penalty0 (4):\penalty0 343--370, 1988.

\bibitem[Bach et~al.(2004)Bach, Lanckriet, and Jordan]{bach2004mkl}
F.~R. Bach, G.~R.~G. Lanckriet, and M.~I. Jordan.
\newblock Multiple kernel learning, conic duality, and the {SMO} algorithm.
\newblock In \emph{Proceedings of the 21st International Conference on Machine
  Learning}, pages 41--48, Banff, Canada, 2004.

\bibitem[Belkin et~al.(2006)Belkin, Niyogi, and
  Sindhwani]{Belkin:Niyogi:Sindhwani2006}
M.~Belkin, P.~Niyogi, and V.~Sindhwani.
\newblock Manifold regularization: A geometric framework for learning from
  labeled and unlabeled examples.
\newblock \emph{Journal of Machine Learning Research}, 7:\penalty0 2399--2434,
  2006.

\bibitem[Boyd and Vandenberghe(2004)]{boyd2004co}
S.~P. Boyd and L.~Vandenberghe.
\newblock \emph{Convex Optimization}.
\newblock Cambridge University Press, Cambridge, UK, 2004.

\bibitem[Bucak et~al.(2011)Bucak, Jin, and Jain]{bucak2011multi}
S.~S. Bucak, R.~Jin, and A.~K. Jain.
\newblock Multi-label learning with incomplete class assignments.
\newblock In \emph{Proceedings of International Conference on Computer Vision
  and Pattern Recognition}, pages 2801--2808, Colorado Springs, CO, 2011.

\bibitem[Chapelle(2007)]{chapelle2007training}
O.~Chapelle.
\newblock Training a support vector machine in the primal.
\newblock \emph{Neural Computation}, 19\penalty0 (5):\penalty0 1155--1178,
  2007.

\bibitem[Chapelle and Zien(2005)]{chapelle2004semi}
O.~Chapelle and A.~Zien.
\newblock Semi-supervised classification by low density separation.
\newblock In \emph{Proceedings of the 10th International Workshop on Artificial
  Intelligence and Statistics}, The Savannah Hotel, Barbados, 2005.

\bibitem[Chapelle et~al.(2006{\natexlab{a}})Chapelle, Chi, and
  Zien]{chapelle2006continuation}
O.~Chapelle, M.~Chi, and A.~Zien.
\newblock A continuation method for semi-supervised {SVMs}.
\newblock In \emph{Proceedings of the 23rd International Conference on Machine
  Learning}, pages 185--192, Pittsburgh, PA, 2006{\natexlab{a}}.

\bibitem[Chapelle et~al.(2006{\natexlab{b}})Chapelle, Sch{\"o}lkopf, and
  Zien]{chapelle2006semi}
O.~Chapelle, B.~Sch{\"o}lkopf, and A.~Zien.
\newblock \emph{Semi-Supervised Learning}.
\newblock MIT Press, Cambridge, MA, USA, 2006{\natexlab{b}}.

\bibitem[Chapelle et~al.(2007)Chapelle, Sindhwani, and Keerthi]{chapelle2006bb}
O.~Chapelle, V.~Sindhwani, and S.~S. Keerthi.
\newblock Branch and bound for semi-supervised support vector machines.
\newblock In B.~Sch\"{o}lkopf, J.~Platt, and T.~Hoffman, editors,
  \emph{Advances in Neural Information Processing Systems 19}, pages 217--224.
  MIT Press, Cambridge, MA, 2007.

\bibitem[Chapelle et~al.(2008)Chapelle, Sindhwani, and Keerthi]{Chapelle2008}
O.~Chapelle, V.~Sindhwani, and S.~S. Keerthi.
\newblock Optimization techniques for semi-supervised support vector machines.
\newblock \emph{Journal of Machine Learning Research}, 9:\penalty0 203--233,
  2008.

\bibitem[Cheung and Kwok(2006)]{cheung2006rfm}
P.~M. Cheung and J.~T. Kwok.
\newblock A regularization framework for multiple-instance learning.
\newblock In \emph{Proceedings of the 23th International Conference on Machine
  Learning}, pages 193--200, Pittsburgh, PA, USA, 2006.

\bibitem[Collobert et~al.(2006)Collobert, Sinz, Weston, and
  Bottou]{collobert2006lst}
R.~Collobert, F.~Sinz, J.~Weston, and L.~Bottou.
\newblock {Large scale transductive SVMs}.
\newblock \emph{Journal of Machine Learning Research}, 7:\penalty0 1687--1712,
  2006.

\bibitem[Crammer and Singer(2002)]{crammer2002algorithmic}
K.~Crammer and Y.~Singer.
\newblock {On the algorithmic implementation of multiclass kernel-based vector
  machines}.
\newblock \emph{Journal of Machine Learning Research}, 2:\penalty0 265--292,
  2002.

\bibitem[Cristianini et~al.(2002)Cristianini, Shawe-Taylor, Elisseeff, and
  Kandola]{cristianini2002kernel}
N.~Cristianini, J.~Shawe-Taylor, A.~Elisseeff, and J.~Kandola.
\newblock On kernel-target alignment.
\newblock In T.~G. Dietterich, Z.~Becker, and Z.~Ghahramani, editors,
  \emph{Advances in Neural Information Processing Systems 14}, pages 367--373.
  MIT Press, Cambridge, MA, 2002.

\bibitem[De~Bie and Cristianini(2006)]{de2006semi}
T.~De~Bie and N.~Cristianini.
\newblock Semi-supervised learning using semi-definite programming.
\newblock In O.~Chapelle, B.~Sch{\"o}lkopf, and A.~Zien, editors,
  \emph{Semi-Supervised Learning}. MIT Press, Cambridge, MA, 2006.

\bibitem[Demsar(2006)]{demvsar2006statistical}
Janez Demsar.
\newblock Statistical comparisons of classifiers over multiple data sets.
\newblock \emph{Journal of Machine Learning Research}, 7:\penalty0 1--30, 2006.

\bibitem[Dietterich et~al.(1997)Dietterich, Lathrop, and
  Lozano-P{\'e}rez]{dietterich1997smi}
T.~G. Dietterich, R.~H. Lathrop, and T.~Lozano-P{\'e}rez.
\newblock Solving the multiple instance problem with axis-parallel rectangles.
\newblock \emph{Artificial Intelligence}, 89\penalty0 (1-2):\penalty0 31--71,
  1997.

\bibitem[Fan et~al.(2005)Fan, Chen, and Lin]{fan2005wss}
R.E. Fan, P.H. Chen, and C.J. Lin.
\newblock {Working set selection using second order information for training
  support vector machines}.
\newblock \emph{Journal of Machine Learning Research}, 6:\penalty0 1889--1918,
  2005.

\bibitem[G{\"a}rtner et~al.(2002)G{\"a}rtner, Flach, Kowalczyk, and
  Smola]{gartner2002mik}
T.~G{\"a}rtner, P.~A. Flach, A.~Kowalczyk, and A.~J. Smola.
\newblock Multi-instance kernels.
\newblock In \emph{Proceedings of the 19th International Conference on Machine
  Learning}, pages 179--186, Sydney, Australia, 2002.

\bibitem[Guo(2009)]{guo2009max}
Y.~Guo.
\newblock Max-margin multiple-instance learning via semidefinite programming.
\newblock In \emph{Proceedings of the 1st Asian Conference on Machine
  Learning}, pages 98--108, Nanjing, China, 2009.

\bibitem[Horst and Thoai(1999)]{horst1999dc}
R.~Horst and N.V. Thoai.
\newblock {DC} programming: Overview.
\newblock \emph{Journal of Optimization Theory and Applications}, 103\penalty0
  (1):\penalty0 1--43, 1999.

\bibitem[Hsieh et~al.(2008)Hsieh, Chang, Lin, Keerthi, and
  Sundararajan]{hsieh2008dcd}
C.~J. Hsieh, K.~W. Chang, C.~J. Lin, S.~S. Keerthi, and S.~Sundararajan.
\newblock A dual coordinate descent method for large-scale linear {SVM}.
\newblock In \emph{Proceedings of the 25th International Conference on Machine
  Learning}, pages 408--415, Helsinki, Finland, 2008.

\bibitem[Jain and Dubes(1988)]{jain1988acd}
A.~K. Jain and R.~C. Dubes.
\newblock \emph{Algorithms for Clustering Data}.
\newblock Prentice Hall, Upper Saddle River, NJ, USA, 1988.

\bibitem[Joachims(1999)]{Joachims1999}
T.~Joachims.
\newblock Transductive inference for text classification using support vector
  machines.
\newblock In \emph{Proceedings of the 16th International Conference on Machine
  Learning}, pages 200--209, Bled, Slovenia, 1999.

\bibitem[Joachims(2006)]{joachims2006tls}
T.~Joachims.
\newblock Training linear {SVMs} in linear time.
\newblock In \emph{Proceedings of the 12th International Conference on
  Knowledge Discovery and Data mining}, pages 217--226, Philadelphia, PA, 2006.

\bibitem[Kelley(1960)]{kelley1960cpm}
J.~E. Kelley.
\newblock The cutting plane method for solving convex programs.
\newblock \emph{Journal of the Society for Industrial and Applied Mathematics},
  8\penalty0 (4):\penalty0 703--712, 1960.

\bibitem[Kim and Boyd(2008)]{kim08}
S.-J. Kim and S.~Boyd.
\newblock A minimax theorem with applications to machine learning, signal
  processing, and finance.
\newblock \emph{SIAM Journal on Optimization}, 19\penalty0 (3):\penalty0
  1344--1367, 2008.

\bibitem[Kloft et~al.(2009)Kloft, Brefeld, Sonnenburg, Laskov, M\"{u}ller, and
  Zien]{kloft2009efficient}
M.~Kloft, U.~Brefeld, S.~Sonnenburg, P.~Laskov, K.-R. M\"{u}ller, and A.~Zien.
\newblock Efficient and accurate lp-norm multiple kernel learning.
\newblock In Y.~Bengio, D.~Schuurmans, J.~Lafferty, C.~K.~I. Williams, and
  A.~Culotta, editors, \emph{Advances in Neural Information Processing Systems
  22}, pages 997--1005. MIT Press, Cambridge, MA, 2009.

\bibitem[Lanckriet et~al.(2004)Lanckriet, Cristianini, Bartlett, Ghaoui, and
  Jordan]{lanckriet2004lkm}
G.~R.~G. Lanckriet, N.~Cristianini, P.~Bartlett, L.~El Ghaoui, and M.~I.
  Jordan.
\newblock Learning the kernel matrix with semidefinite programming.
\newblock \emph{Journal of Machine Learning Research}, 5:\penalty0 27--72,
  2004.

\bibitem[Li and Zhou(2011)]{li2011s4vm}
Y.-F. Li and Z.-H. Zhou.
\newblock Towards making unlabeled data never hurt.
\newblock In \emph{Proceedings of the 28th International Conference on Machine
  learning}, pages 1081--1088, Bellevue, WA, 2011.

\bibitem[Li et~al.(2009{\natexlab{a}})Li, Kwok, Tsang, and Zhou]{li2009convex}
Y.-F. Li, J.T. Kwok, I.W. Tsang, and Z.-H. Zhou.
\newblock A convex method for locating regions of interest with multi-instance
  learning.
\newblock In \emph{Proceedings of the 20th European Conference on Machine
  Learning and Knowledge Discovery in Databases}, pages 15--30, Bled, Slovenia,
  2009{\natexlab{a}}.

\bibitem[Li et~al.(2009{\natexlab{b}})Li, Kwok, and Zhou]{li2009semi}
Y.-F. Li, J.T. Kwok, and Z.-H. Zhou.
\newblock Semi-supervised learning using label mean.
\newblock In \emph{Proceedings of the 26th International Conference on Machine
  Learning}, pages 633--640, Montreal, Canada, 2009{\natexlab{b}}.

\bibitem[Li et~al.(2009{\natexlab{c}})Li, Tsang, Kwok, and Zhou]{li2009tighter}
Y.-F. Li, I.W. Tsang, J.T. Kwok, and Z.-H. Zhou.
\newblock Tighter and convex maximum margin clustering.
\newblock In \emph{Proceedings of the 12th International Conference on
  Artificial Intelligence and Statistics}, pages 344--351, Clearwater Beach,
  FL, 2009{\natexlab{c}}.

\bibitem[Li et~al.(2012)Li, Hu, Jiang, and Zhou]{li2012}
Y.-F. Li, J.-H. Hu, Y.~Jiang, and Z.-H. Zhou.
\newblock Towards discovering what patterns trigger what labels.
\newblock In \emph{Proceedings of the 26th AAAI Conference on Artificial
  Intelligence}, pages 1012--1018, Toronto, Canada, 2012.

\bibitem[Lobo et~al.(1998)Lobo, Vandenberghe, Boyd, and
  Lebret]{lobo1998applications}
M.S. Lobo, L.~Vandenberghe, S.~Boyd, and H.~Lebret.
\newblock Applications of second-order cone programming.
\newblock \emph{Linear algebra and its applications}, 284\penalty0
  (1):\penalty0 193--228, 1998.

\bibitem[Maron and Ratan(1998)]{maron1998mil}
O.~Maron and A.~L. Ratan.
\newblock Multiple-instance learning for natural scene classification.
\newblock In \emph{Proceedings of the 15th International Conference on Machine
  Learning}, pages 341--349, Madison, WI, 1998.

\bibitem[Mitchell(2006)]{mitchell2006discipline}
T.M. Mitchell.
\newblock The discipline of machine learning.
\newblock Technical report, Machine Learning Department, Carnegie Mellon
  University, 2006.

\bibitem[Nesterov and Nemirovskii(1987)]{nesterov1987interior}
Y.~Nesterov and A.~Nemirovskii.
\newblock \emph{Interior-Point Polynomial Algorithms in Convex Programming},
  volume~13.
\newblock Society for Industrial Mathematics, 1987.

\bibitem[Platt(1999)]{platt-99b}
J.~C. Platt.
\newblock Fast training of support vector machines using sequential minimal
  optimization.
\newblock In \emph{Advances in Kernel Methods - {S}upport Vector Learning},
  pages 185--208. MIT Press, Cambridge, MA, USA, 1999.

\bibitem[Rakotomamonjy et~al.(2008)Rakotomamonjy, Bach, Canu, and
  Grandvalet]{rakotomamonjy-08}
A.~Rakotomamonjy, F.~R. Bach, S.~Canu, and Y.~Grandvalet.
\newblock Simple{MKL}.
\newblock \emph{Journal of Machine Learning Research}, 9:\penalty0 2491--2521,
  2008.

\bibitem[Sch{\"o}lkopf and Smola(2002)]{sch?lkopf2002learning}
B.~Sch{\"o}lkopf and A.J. Smola.
\newblock \emph{Learning with Kernels}.
\newblock MIT Press, 2002.

\bibitem[Shalev-Shwartz et~al.(2007)Shalev-Shwartz, Singer, and
  Srebro]{shalev2007pegasos}
S.~Shalev-Shwartz, Y.~Singer, and N.~Srebro.
\newblock Pegasos: Primal estimated sub-gradient solver for svm.
\newblock In \emph{Proceedings of the 24th International Conference on Machine
  Learning}, pages 807--814, Corvallis, OR, 2007.

\bibitem[Sheng et~al.(2008)Sheng, Provost, and Ipeirotis]{sheng2008get}
V.S. Sheng, F.~Provost, and P.G. Ipeirotis.
\newblock Get another label? improving data quality and data mining using
  multiple, noisy labelers.
\newblock In \emph{Proceedings of the 14th International Conference on
  Knowledge Discovery and Data mining}, pages 614--622, Las Vegas, NV, 2008.

\bibitem[Shi and Malik(2000)]{shi2000nca}
J.~Shi and J.~Malik.
\newblock Normalized cuts and image segmentation.
\newblock \emph{IEEE Transactions on Pattern Analysis and Machine
  Intelligence}, 22\penalty0 (8):\penalty0 888--905, 2000.

\bibitem[Sindhwani and Keerthi(2006)]{sindhwani2006large}
V.~Sindhwani and S.S. Keerthi.
\newblock Large scale semi-supervised linear {SVM}s.
\newblock In \emph{Proceedings of the 29th annual International Conference on
  Research and Development in Information Retrieval}, pages 477--484, Seattle,
  WA, 2006.

\bibitem[Sindhwani et~al.(2006)Sindhwani, Keerthi, and
  Chapelle]{sindhwani2006das}
V.~Sindhwani, S.S. Keerthi, and O.~Chapelle.
\newblock {Deterministic annealing for semi-supervised kernel machines}.
\newblock In \emph{Proceedings of the 23rd International Conference on Machine
  Learning}, pages 841--848, Pittsburgh, PA, 2006.

\bibitem[Sonnenburg et~al.(2006)Sonnenburg, R{\"a}tsch, Sch{\"a}fer, and
  Sch{\"o}lkopf]{sonnenburg2006lsm}
S.~Sonnenburg, G.~R{\"a}tsch, C.~Sch{\"a}fer, and B.~Sch{\"o}lkopf.
\newblock Large scale multiple kernel learning.
\newblock \emph{Journal of Machine Learning Research}, 7:\penalty0 1531--1565,
  2006.

\bibitem[Subramanya and Bilmes(2009)]{subramanya2009entropic}
A.~Subramanya and J.~Bilmes.
\newblock Entropic graph regularization in non-parametric semi-supervised
  classification.
\newblock In Y.~Bengio, D.~Schuurmans, J.~Lafferty, C.~K.~I. Williams, and
  A.~Culotta, editors, \emph{Advances in Neural Information Processing Systems
  22}, pages 1803--1811. MIT Press, Cambridge, MA, 2009.

\bibitem[Sun et~al.(2010)Sun, Zhang, and Zhou]{sun2010multi}
Y.-Y. Sun, Y.~Zhang, and Z.-H. Zhou.
\newblock Multi-label learning with weak label.
\newblock In \emph{Proceedings of the 24th AAAI Conference on Artificial
  Intelligence}, pages 593--598, Atlanta, GA, 2010.

\bibitem[Tsang et~al.(2006)Tsang, Kwok, and Cheung]{tsang2006cvm}
I.~W. Tsang, J.~T. Kwok, and P.~Cheung.
\newblock Core vector machines: Fast {SVM} training on very large data sets.
\newblock \emph{Journal of Machine Learning Research}, 6:\penalty0 363--392,
  2006.

\bibitem[Tsochantaridis et~al.(2006)Tsochantaridis, Joachims, Hofmann, and
  Altun]{tsochantaridis2006large}
I.~Tsochantaridis, T.~Joachims, T.~Hofmann, and Y.~Altun.
\newblock Large margin methods for structured and interdependent output
  variables.
\newblock \emph{Journal of Machine Learning Research}, 6\penalty0 (2):\penalty0
  1453, 2006.

\bibitem[Valizadegan and Jin(2007)]{valizadegan1400gmm}
H.~Valizadegan and R.~Jin.
\newblock Generalized maximum margin clustering and unsupervised kernel
  learning.
\newblock In B.~Sch\"{o}lkopf, J.~Platt, and T.~Hoffman, editors,
  \emph{Advances in Neural Information Processing Systems 19}, pages
  1417--1424. MIT Press, Cambridge, MA, 2007.

\bibitem[Vapnik(1998)]{vapnik1998statistical}
V.N. Vapnik.
\newblock \emph{Statistical Learning Theory}.
\newblock Wiley-Interscience, 1998.

\bibitem[Wang et~al.(2008)Wang, Yang, and Zha]{wang2008app}
H.~Y. Wang, Q.~Yang, and H.~Zha.
\newblock {Adaptive p-posterior mixture-model kernels for multiple instance
  learning}.
\newblock In \emph{Proceedings of the 25th International Conference on Machine
  Learning}, pages 1136--1143, Helsinki, Finland, 2008.

\bibitem[Xu and Schuurmans(2005)]{xu2005semi}
L.~Xu and D.~Schuurmans.
\newblock Unsupervised and semi-supervised multi-class support vector machines.
\newblock In \emph{Proceedings of the 20th National Conference on Artificial
  Intelligence}, pages 904--910, Pittsburgh, PA, 2005.

\bibitem[Xu et~al.(2005)Xu, Neufeld, Larson, and Schuurmans]{xu2005mmc}
L.~Xu, J.~Neufeld, B.~Larson, and D.~Schuurmans.
\newblock Maximum margin clustering.
\newblock In L.~K. Saul, Y.~Weiss, and L.~Bottou, editors, \emph{Advances in
  Neural Information Processing Systems 17}, pages 1537--1544. MIT Press,
  Cambridge, MA, 2005.

\bibitem[Xu and Frank(2004)]{xu2004lra}
X.~Xu and E.~Frank.
\newblock {Logistic regression and boosting for labeled bags of instances}.
\newblock In \emph{Proceedings of the 8th Pacific-Asia Conference on Knowledge
  Discovery and Data Mining}, pages 272--281, Sydney, Australia, 2004.

\bibitem[Xu et~al.(2009)Xu, Jin, King, and Lyu]{xu-09}
Z.~Xu, R.~Jin, I.~King, and M.~R. Lyu.
\newblock An extended level method for efficient multiple kernel learning.
\newblock In D.~Koller, D.~Schuurmans, Y.~Bengio, and L.~Bottou, editors,
  \emph{Advances in Neural Information Processing Systems 21}, pages
  1825--1832. MIT Press, Cambridge, MA, 2009.

\bibitem[Xu et~al.(2010)Xu, Jin, Yang, King, and Lyu]{xusimple}
Z.~Xu, R.~Jin, H.~Yang, I.~King, and M.~Lyu.
\newblock Simple and efficient multiple kernel learning by group lasso.
\newblock In \emph{Proceedings of 27th International Conference on Machine
  Learning}, pages 1--8, Haifa, Israel, 2010.

\bibitem[Yang et~al.(2013)Yang, Jiang, and Zhou]{yangIJCAI13}
S.-J. Yang, Y.~Jiang, and Z.-H. Zhou.
\newblock Multi-instance multi-label learning with weak label.
\newblock In \emph{Proceedings of 23rd International Joint Conference on
  Artificial Intelligence}, Beijing, China, 2013.

\bibitem[Zhang et~al.(2007)Zhang, Tsang, and Kwok]{zhang2007mmc}
K.~Zhang, I.~W. Tsang, and J.~T. Kwok.
\newblock Maximum margin clustering made practical.
\newblock In \emph{Proceedings of the 24th International Conference on Machine
  learning}, pages 1119--1126, Corvallis, OR, 2007.

\bibitem[Zhang et~al.(2009{\natexlab{a}})Zhang, Kwok, and
  Parvin]{zhang2009prototype}
K.~Zhang, J.T. Kwok, and B.~Parvin.
\newblock Prototype vector machine for large scale semi-supervised learning.
\newblock In \emph{Proceedings of the 26th International Conference on Machine
  Learning}, pages 1233--1240, Montreal, Canada, 2009{\natexlab{a}}.

\bibitem[Zhang et~al.(2009{\natexlab{b}})Zhang, Tsang, and Kwok]{zhang2009mmc}
K.~Zhang, I.~W. Tsang, and J.~T. Kwok.
\newblock Maximum margin clustering made practical.
\newblock \emph{IEEE Transactions on Neural Networks}, 20\penalty0
  (4):\penalty0 583--596, 2009{\natexlab{b}}.

\bibitem[Zhang and Goldman(2002)]{zhang2002dim}
Q.~Zhang and S.~A. Goldman.
\newblock {EM-DD}: An improved multiple-instance learning technique.
\newblock In T.~G. Dietterich, S.~Becker, and Z.~Ghahramani, editors,
  \emph{Advances in Neural Information Processing Systems 14}, pages
  1073--1080. MIT Press, Cambridge, MA, 2002.

\bibitem[Zhao et~al.(2008)Zhao, Wang, and Zhang]{zhao:emm}
B.~Zhao, F.~Wang, and C.~Zhang.
\newblock Efficient maximum margin clustering via cutting plane algorithm.
\newblock In \emph{Proceedings of the 8th International Conference on Data
  Mining}, pages 751--762, Atlanta, GA, 2008.

\bibitem[Zhou and Li(2010)]{Zhou:Li2010}
Z.-H. Zhou and M.~Li.
\newblock Semi-supervised learning by disagreement.
\newblock \emph{Knowledge and Information Systems}, 24\penalty0 (3):\penalty0
  415--439, 2010.

\bibitem[Zhou and Xu(2007)]{zhou2007relation}
Z.-H. Zhou and J.-M. Xu.
\newblock On the relation between multi-instance learning and semi-supervised
  learning.
\newblock In \emph{Proceedings of the 24th International Conference on Machine
  Learning}, pages 1167--1174, Corvallis, OR, 2007.

\bibitem[Zhou and Zhang(2007)]{Zhou:Zhang2007nips06}
Z.-H. Zhou and M.-L. Zhang.
\newblock Multi-instance multi-label learning with application to scene
  classification.
\newblock In B.~Sch{\"o}lkopf, J.~Platt, and T.~Hofmann, editors,
  \emph{Advances in Neural Information Processing Systems 19}, pages
  1609--1616. MIT Press, Cambridge, MA, 2007.

\bibitem[Zhou et~al.(2005)Zhou, Xue, and Jiang]{zhou2005lri}
Z.-H. Zhou, X.-B. Xue, and Y.~Jiang.
\newblock Locating regions of interest in {CBIR} with multi-instance learning
  techniques.
\newblock In \emph{Proceedings of the 18th Australian Joint Conference on
  Artificial Intelligence}, pages 92--101, Sydney, Australia, 2005.

\bibitem[Zhou et~al.(2012)Zhou, Zhang, Huang, and Li]{zhou2012multi}
Z.-H. Zhou, M.-L. Zhang, S.-J. Huang, and Y.-F. Li.
\newblock Multi-instance multi-label learning.
\newblock \emph{Artificial Intelligence}, 176\penalty0 (1):\penalty0
  2291--2320, 2012.

\bibitem[Zhu(2006)]{zhu2006semi}
X.~Zhu.
\newblock Semi-supervised learning literature survey.
\newblock Technical report, Computer Science, University of Wisconsin-Madison,
  2006.

\end{thebibliography}

\end{document}